\documentclass{article}

\usepackage[final]{neurips_2023}

\usepackage{raunak}
\usepackage{accents}
\usepackage{thmtools,thm-restate}
\usepackage{wrapfig}
\usepackage{tikz}
\usetikzlibrary{decorations.pathreplacing,positioning, arrows.meta}

\usepackage{url}            
\usepackage{microtype}      

\title{Online Convex Optimization with Unbounded Memory}

\author{%
  Raunak Kumar \\
  Department of Computer Science \\
  Cornell University \\
  Ithaca, NY 14853 \\
  \texttt{raunak@cs.cornell.edu}
  \And
  Sarah Dean \\
  Department of Computer Science \\
  Cornell University \\
  Ithaca, NY 14853 \\
  \texttt{sdean@cornell.edu}
  \And
  Robert Kleinberg \\
  Department of Computer Science \\
  Cornell University \\
  Ithaca, NY 14853 \\
  \texttt{rdk@cs.cornell.edu}
}

\newcommand{\circfn}[1]{\tilde{#1}}

\begin{document}

\maketitle

\begin{abstract}
  Online convex optimization (OCO) is a widely used framework in online
  learning. In each round, the learner chooses a decision in a convex set and an
  adversary chooses a convex loss function, and then the learner suffers the
  loss associated with their current decision. However, in many applications the
  learner's loss depends not only on the current decision but on the entire
  history of decisions until that point. The OCO framework and its existing
  generalizations do not capture this, and they can only be applied to many
  settings of interest after a long series of approximation arguments. They also
  leave open the question of whether the dependence on memory is tight because
  there are no non-trivial lower bounds. In this work we introduce a
  generalization of the OCO framework, ``Online Convex Optimization with
  Unbounded Memory'', that captures long-term dependence on past decisions. We
  introduce the notion of $p$-effective memory capacity, $H_p$, that quantifies
  the maximum influence of past decisions on present losses. We prove an
  $O(\sqrt{H_p T})$ upper bound on the policy regret and a matching (worst-case)
  lower bound. As a special case, we prove the first non-trivial lower bound for
  OCO with finite memory~\citep{anavaHM2015online}, which could be of
  independent interest, and also improve existing upper bounds. We demonstrate
  the broad applicability of our framework by using it to derive regret bounds,
  and to improve and simplify existing regret bound derivations, for a variety
  of online learning problems including online linear control and an online
  variant of performative prediction.
\end{abstract}


\section{Introduction}
\label{sec:intro}


Numerous applications are characterized by multiple rounds of sequential
interactions with an environment, e.g., prediction from expert
advice~\citep{littlestoneW1989weighted,littlestoneW1994weighted}, portoflio
selection~\citep{cover1991universal}, routing~\citep{awerbuchK2008online}, etc.
One of the most popular frameworks for modelling such sequential decision-making
problems is online convex optimization (OCO)~\citep{zinkevich2003online}. The
OCO framework is as follows. In each round, the learner chooses a decision in a
convex set and an adversary chooses a convex loss function, and then the learner
suffers the loss associated with their current decision. The performance of an
algorithm is measured by regret: the difference between the algorithm's total
loss and that of the best fixed decision. We refer the reader
to~\citet{shalevshwartz2012online,hazan2019introduction,orabona2019modern} for
surveys on this topic.

However, in many applications the loss of the learner depends not only on the
current decisions but on the entire history of decisions until that point. For
example, in online linear control~\citep{agarwalBHKS2019online}, in each round
the learner chooses a ``control policy'' (i.e., decision), suffers a loss that
is a function of the action taken by this policy and the current state of the
system, and the system's state evolves according to linear dynamics. The current
state depends on the entire history of actions and, therefore, the current loss
depends not only on the current decision but the entire history of decisions.
The OCO framework cannot capture such long-term dependence of the current loss
on the past decisions and neither can existing generalizations that allow the
loss to depend on a \emph{constant} number of past
decisions~\citep{anavaHM2015online}. Although a series of approximation
arguments can be used to apply finite memory generalizations of OCO to the
online linear control problem, there is no OCO framework that captures the
complete long-term dependence of current losses on past decisions. Furthermore,
there are no non-trivial lower bounds for OCO in the memory
setting,\footnote{The trivial lower bound refers to the $\Omega(\sqrt{T})$ lower
bound for OCO in the memoryless setting.} which leaves open the question whether
the dependence on memory is tight.

\paragraph{Contributions.}
In this paper we introduce a generalization of the OCO framework, ``Online
Convex Optimization with Unbounded Memory''~(\cref{sec:framework}), that allows
the loss in the current round to depend on the entire history of decisions until
that point. We introduce the notion of $p$-effective memory capacity, $H_p$,
that quantifies the maximum influence of past decisions on present losses. We
prove an $O(\sqrt{H_p T})$ upper bound on the policy
regret~(\cref{thm:regret_upper_bound}) and a matching (worst-case) lower
bound~(\cref{thm:regret_lower_bound}). As a special case, we prove the first
non-trivial lower bound for OCO with finite
memory~(\cref{thm:regret_lower_bound_finite}), which could be of independent
interest, and also improve existing upper
bounds~(\cref{thm:regret_upper_bound_finite}). Our lower bound technique extends
existing techniques developed for memoryless settings. We design novel
adversarial loss functions that exploit the fact that an algorithm cannot
overwrite its history.
We illustrate the power of our framework by bringing together the regret
analysis of two seemingly disparate problems under the same umbrella. First, we
show how our framework improves and simplifies existing regret bounds for the
online linear control problem~\citep{agarwalBHKS2019online}
in~\cref{thm:regret_upper_bound_olc}. Second, we show how our framework can be
used to derive regret bounds for an online variant of performative
prediction~\citep{perdomoZMH2020performative}
in~\cref{thm:regret_upper_bound_opp}. This demonstrates the broad applicability
of our framework for deriving regret bounds for a variety of online learning
problems, particularly those that exhibit long-term dependence of current losses
on past decisions.

\paragraph{Related work.}
The most closely related work to ours is the OCO with finite memory
framework~\citep{anavaHM2015online}. They consider a generalization of the OCO
framework that allows the current loss to depend on a \emph{constant} number of
past decisions. There have been a number of follow-up works that extend the
framework in a variety of other ways, such as
non-stationarity~\citep{zhaoWZ2022nonstationary}, incorporating switching
costs~\citep{shiLCYW2020online}, etc. However, none of these existing works go
beyond a constant memory length and do not prove a non-trivial lower bound with
a dependence on the memory length.  In a different line of
work,~\citet{bhatiaS2020online} consider a much more general online learning
framework that goes beyond a constant memory length, but they only provide
\emph{non-constructive} upper bounds on regret. In contrast, our OCO with
unbounded memory framework allows the current loss to depend on an
\emph{unbounded} number of past decisions, provides \emph{constructive} upper
bounds on regret, and lower bounds for a broad class of problems that includes
OCO with finite memory with a general memory length $m$.

A different framework for sequential decision-making is multi-armed
bandits~\citep{bubeckC2012regret,slivkins2019introduction}.~\citet{qinLPFO2023stochastic}
study a variant of contextual stochastic bandits where the current loss can
depend on a sparse subset of all prior contexts. This setting differs from ours
due to the feedback model, stochasticity, and decision space.
Reinforcement learning~\citep{suttonB2018reinforcement} is yet another popular
framework for sequential decision-making that considers very general
state-action models of feedback and dynamics. In reinforcement learning one
typically measures regret with respect to the best state-action policy from some
policy class, rather than the best fixed decision as in online learning and OCO.
In the special case of linear control, policies can be reformulated as decisions
while preserving convexity; we discuss this application
in~\cref{sec:applications}. Considering the general framework is an active area
of research.

We defer discussion of related work for specific applications
to~\cref{sec:applications}.

\section{Framework}
\label{sec:framework}



We begin with some motivation for the formalism used in our
framework~(\cref{subsec:framework_setup}). Many real-world applications involve
controlling a physical dynamical system, for example, variable-speed wind
turbines in wind energy electric power
production~\citep{boukhezzar2010comparison}.
The typical solution for these problems has been to model them as offline
control problems with linear time-invariant dynamics and use classical methods
such as LQR and LQG~\citep{boukhezzar2010comparison}. Instead of optimizing over
the space of control inputs, the typical feedback control approach optimizes
over the space of controllers, i.e., policies that choose a control input as a
function of the system state. The standard controllers considered in the
literature are linear controllers. Even when the losses are convex in the state
and input, they are nonconvex in the linear controller. In the special case of
quadratic losses in terms of the state and input, there is a closed-form
solution for the optimal solution using the algebraic Riccati
equations~\citep{lancasterR1995algebraic}. But this does not hold for general
convex losses resulting in convex reparameterizations such as
Youla~\citep{youlaJB1976modern,kuvcera1975stability} and
SLS~\citep{wangMD2019system,andersonDLM2019system}. The resulting
parameterization represents an infinite dimensional system response and is
characterized by a sequence of matrices. Recent work has studied an online
approach for some of these control theory problems, where a sequence of
controllers is chosen adaptively rather than choosing one
offline~\citep{abbasi-yadkoriS2011regret,deanMMRT2018regret,simchowitzF2020naive,agarwalBHKS2019online}.


The takeaway from the above is that there are online learning problems in which
(i) the current loss depends on the entire history of decisions; and (ii) the
decision space can be more complicated than just a subset of $\R^d$, e.g., it
can be an unbounded sequence of matrices. This motivates us to model the
decision space as a Hilbert space and the history space as a Banach space in the
formal problem setup below, and this subsumes the special cases of OCO and OCO
with finite memory. This formalism not only lets us consider a wide range of
spaces, such as $\R^d$, unbounded sequences of matrices, etc., but also lets us
define appropriate norms on these spaces. This latter feature is crucial for
deriving strong regret bounds for some applications such as online linear
control. For this problem we derive improved regret
bounds~(\cref{thm:regret_upper_bound_olc}) by defining weighted norms on the
decision and history spaces, where the weights are chosen to leverage the
problem structure.

\paragraph{Notation.}
We use $\| \cdot \|_\calU$ to denote the norm associated with a space $\calU$.
The operator norm for a linear $L$ operator from space $\calU \to \calV$ is
defined as $\|L\|_{\calU \to \calV} = \max_{u : \|u\|_\calU \leq 1}
\|Lu\|_\calV$. For convenience, sometimes we simply use $\| \cdot \|$ when the
meaning is clear from the context. For a finite-dimensional matrix we use $\|
\cdot \|_F$ and $\| \cdot \|_2$ to denote its Frobenius norm and operator norm
respectively.


\subsection{Setup}
\label{subsec:framework_setup}

Let the decision space $\calX$ be a closed and convex subset of a Hilbert space
$\calW$ with norm $\| \cdot \|_\calX$ and the history space $\calH$ be a Banach
space with norm $\| \cdot \|_\calH$. Let $A : \calH \to \calH$ and $B : \calW
\to \calH$ be linear operators. The game between the learner and an oblivious
adversary proceeds as follows. Let $T$ denote the time horizon and $f_t : \calH
\to \R$ be loss functions chosen by the adversary. The initial history is $h_0 =
0$. In each round $t \in [T]$, the learner chooses $x_t \in \calX$, the history
is updated to $h_t = A h_{t-1} + B x_t$, and the learner suffers loss
$f_t(h_t)$. An instance of an online convex optimization with unbounded memory
problem is specified by the tuple $(\calX, \calH, A, B)$.

We use the notion of policy regret~\citep{dekelTA2012online} as the performance
measure in our framework. The policy regret of a learner is the difference
between its total loss and the total loss of a strategy that plays the best
fixed decision in every round. The history after round $t$ for a strategy that
chooses $x$ in every round is described by $h_t = \sum_{k=0}^{t-1} A^k B x$,
which motivates the following definition.

\begin{definition}\label{def:circfn}
  Given $f_t : \calH \to \R$, the function $\circfn{f}_t : \calX \to \R$ is
  defined by $\circfn{f}_t(x) = f_t(\sum_{k=0}^{t-1} A^k B x)$.
\end{definition}

\begin{definition}[Policy Regret]\label{def:regret}
  The policy regret of an algorithm $\calA$ is defined as $R_T(\calA) =
  \sum_{t=1}^T f_t(h_t) - \min_{x \in \calX} \sum_{t=1}^T \circfn{f_t}(x)$.
\end{definition}

In many motivating examples such as online linear
control~(\cref{subsec:applications_olc}), the history at the end of a round is a
sequence of linear transformations of past decisions. The following definition
captures this formally and we leverage this structure to prove stronger regret
bounds~(\cref{thm:regret_upper_bound}).

\begin{definition}[Linear Sequence Dynamics]\label{def:sequence}
  Consider an online convex optimization with unbounded memory problem specified
  by $(\calX, \calH, A, B)$. Let $(\xi_k)_{k=0}^\infty$ be a sequence of
  nonnegative real numbers satisfying $\xi_0 = 1$. We say that $(\calX, \calH,
  A, B)$ follows linear sequence dynamics with the $\xi$-weighted $p$-norm for
  $p \geq 1$ if
  \begin{enumerate}
    \item $\calH$ is the $\xi$-weighted $\ell^p$-direct sum of a finite or
    countably infinite number of copies of $\calW$: every element $y \in \calH$
    is a sequence $y = (y_i)_{i \in \calI}$, where 
    $\calI = \N$ or $\calI = \{0,\ldots,n\}$ for some $n \in N$, and $\| y
    \|_\calH = \left(\sum_{i \in \calI} (\xi_i \| y_i \|)^p
    \right)^{\nicefrac{1}{p}} < \infty$.
    
    \item We have $A(y_0, y_1, \dots) = (0, A_0 y_0, A_1 y_1, \dots)$, where
    $A_i : \calW \to \calW$ are linear operators.
    
    \item The operator $B$ satisfies $B(x) = (x, 0, \dots)$.
    
  \end{enumerate} 
\end{definition}

Note that since the norm on $\calH$ depends on the weights $\xi$, the operator
norm $\| A^k \|$ also depends on $\xi$. If the weights are all equal to $1$,
then we simply say $p$-norm instead of $\xi$-weighted $p$-norm.


\subsection{Assumptions}
\label{subsec:framework_assumptions}

We make the following assumptions about the feedback model and the loss functions.

\begin{enumerate}[label=\textbf{A\arabic*}]

  \item The learner knows the operators $A$ and $B$, and observes $f_t$ at the end of each round $t$. 
  \label{ass:feedback}
  
  \item The operator norm of $B$ is at most $1$, i.e., $\| B \| \leq 1$. 
  \label{ass:norm_B}
  
  \item The functions $f_t$ are convex. 
  \label{ass:convexity}
  
  \item The functions $f_t$ are $L$-Lipschitz continuous: $\forall \  h,
  \tilde{h} \in \calH$ and $t \in [T]$, we have $|f_t(h) - f_t(\tilde{h})| \leq
  L \| h - \tilde{h} \|_\calH$.
  \label{ass:lipschitz}

\end{enumerate}

Regarding Assumption~\ref{ass:feedback}, our results easily extend to the case
where instead of observing $f_t$, the learner receives a gradient $\nabla
\circfn{f}_t(x_t)$ from a gradient oracle, which can be implemented using
knowledge of $f_t$ and the dynamics $A$ and $B$. Handling the cases when the
operators $A$ and $B$ are unknown and/or the learner observes bandit feedback
(i.e., only $f_t(h_t)$) are important problems and we leave them as future work.
Note that our assumption that $A$ and $B$ are known is no more restrictive than
in the existing literature on OCO with finite memory~\citep{anavaHM2015online}
where it is assumed that the learner knows the constant memory length. In fact,
our assumption is more general because our framework not only captures constant
memory length as a special case but allows for richer dynamics as we illustrate
in~\cref{sec:applications}. Assumption~\ref{ass:norm_B} is made for convenience,
and it amounts to a rescaling of the problem.  Assumption~\ref{ass:convexity}
can be replaced by the \emph{weaker} assumption that $\circfn{f}_t$ are convex
(similar to the literature on OCO with finite memory~\citep{anavaHM2015online})
and this is what we use in the rest of the paper.

Assumptions~\ref{ass:feedback} and~\ref{ass:lipschitz} imply that $\circfn{f}_t$
are $\circfn{L}$-Lipschitz continuous for the following $\circfn{L}$.

\begin{restatable}{theorem}{thmlcirc}\label{thm:lcirc}
  Consider an online convex optimization with unbounded memory problem specified
  by $(\calX, \calH, A, B)$. If $f_t$ is $L$-Lipschitz continuous, then
  $\circfn{f}_t$ is $\circfn{L}$-Lipschitz continuous for $\circfn{L} \leq L
  \sum_{k=0}^\infty \| A^k \|$. If $(\calX, \calH, A, B)$ follows linear
  sequence dynamics with the $\xi$-weighted $p$-norm for $p \geq 1$, then $
  \circfn{L} \leq L \left( \sum_{k=0}^\infty \| A^k \|^p \right)^{\frac{1}{p}}
  $.
\end{restatable}

The proof follows from the definitions of $\circfn{f}_t$ and $\| \cdot
\|_\calH$, and we defer it to \cref{sec:appendix_framework}. The above bound is
tighter than similar results in the literature on OCO with finite memory and
online linear control. This theorem is a key ingredient, amongst others, in
improving existing upper bounds on regret for OCO with finite
memory~(\cref{thm:regret_upper_bound_finite}) and for online linear
control~(\cref{thm:regret_upper_bound_olc}). Before presenting our final
assumption we introduce the notion of $p$-effective memory capacity that
quantifies the maximum influence of past decisions on present losses.

\begin{definition}[$p$-Effective Memory Capacity]\label{def:Hp}
  Consider an online convex optimization with unbounded memory problem specified
  by $(\calX, \calH, A, B)$. For $p \geq 1$, the $p$-effective memory capacity
  is defined as
  \begin{equation}
    H_p(\calX, \calH, A, B) = \left( \sum_{k=0}^\infty k^p \| A^k \|^p \right)^{\frac{1}{p}}.
  \end{equation}
\end{definition}

When the meaning is clear from the context we simply use $H_p$ instead. The
$p$-effective memory capacity is an upper bound on the difference in histories
for two sequences of decisions whose difference grows at most linearly with
time. To see this, consider two sequences of decisions, $(x_k)$ and
$(\tilde{x}_k$), whose elements differ by no more than $k$ at time $k$: $\| x_k
- \tilde{x}_k \| \leq k$. Then the histories generated by the two sequences have
difference between bounded as $\| h - \tilde{h} \| = \| \sum_k A^k B (x_k -
\tilde{x}_k) \| \leq \sum_k k \| A^k B \| \leq \sum_k k \| A^k \| = H_1$, where
the last inequality follows from Assumption~\ref{ass:norm_B}. A similar bound
holds with $H_p$ instead when $(\calX, \calH, A, B)$ follows linear sequence
dynamics with the $\xi$-weighted $p$-norm.

\begin{enumerate}[label=\textbf{A\arabic*}]
  \setcounter{enumi}{4}
  
  \item The $1$-effective memory capacity is finite, i.e., $H_1 < \infty$.
  \label{ass:h1}
  
\end{enumerate}

Since $H_p$ is decreasing in $p$, $H_1 < \infty$ implies $H_p < \infty$ for all
$p \geq 1$. For the case of linear sequence dynamics with the $\xi$-weighted
$p$-norm it suffices to make the \emph{weaker} assumption that $H_p < \infty$.
However, for simplicity of exposition, we assume that $H_1 < \infty$.


\subsection{Special Cases}
\label{subsec:framework_special_cases}

\paragraph{OCO with Finite Memory.}
Consider the OCO with finite memory problem with constant memory length $m$. It
can be specified in our framework by $(\calX, \calH, A_{\text{finite},m},
B_{\text{finite},m})$, where $\calH$ is the $\ell^2$-direct sum of $m$ copies of
$\calX$, $A_{\text{finite},m}(x^{[m]}, \dots, x^{[1]}) = (0, x^{[m]}, \dots,
x^{[2]})$, and $B_{\text{finite},m}(x) = (x, 0, \dots, 0)$. Note that $(\calX,
\calH, A_{\text{finite},m}, B_{\text{finite},m})$ follows linear sequence
dynamics with the $2$-norm. Our framework can even model an extension where the
problem follows linear sequence dynamics with the $p$-norm for $p \geq 1$ by
simply defining $\calH$ to be the $\ell^p$-direct sum of $m$ copies of $\calX$.

\paragraph{OCO with $\rho$-discounted Infinite Memory.}
Our framework can also model OCO with infinite memory problems that are not
modelled by existing OCO frameworks. Let $\rho \in (0, 1)$ be the discount
factor and $p \geq 1$. An OCO with $\rho$-discounted infinite memory problem is
specified by $(\calX, \calH, A_{\text{infinite},\rho},
B_{\text{infinite},\rho})$, where $\calH$ is the $\ell^p$-direct sum of
countably many copies of $\calX$,
$A_{\text{infinite},\rho}((y_0, y_1, \dots)) = (0, \rho y_0, \rho y_1, \dots)$,
and $B_{\text{infinite},\rho}(x) = (x, 0, \dots)$. Note that $(\calX, \calH,
A_{\text{infinite},\rho}, B_{\text{infinite},\rho})$ follows linear sequence
dynamics with the $p$-norm. Due to space constraints we defer proofs of regret
bounds for this problem to the appendix.



\section{Regret Analysis}
\label{sec:regret_analysis}


We present two algorithms for choosing the decisions $x_t$.~\cref{alg:ftrl} uses
follow-the-regularized-leader
(FTRL)~\citep{shalevshwartzS2006online,abernethyHR2008competing} on the loss
functions $\circfn{f}_t$. Due to space constraints, we discuss how to implement
it efficiently in~\cref{sec:appendix_implementation} and present simple
simulation experiments
in~\cref{sec:appendix_experiments}.~\cref{alg:minibatch_ftrl}, which we only
present in~\cref{sec:appendix_minibatchftrl}, combines FTRL with a mini-batching
approach~\citep{dekelTA2012online,altschulerT2018online,chenYLK2020minimax} to
additionally guarantee that the decisions switch at most $O(\nicefrac{T
\circfn{L} }{L H_1})$ times. We defer the proofs of the following upper and
lower bounds
to~\cref{sec:appendix_regret_upper_bounds,sec:appendix_regret_lower_bounds}
respectively.

\begin{algorithm}[ht]
  \caption{\texttt{FTRL}}
  \label{alg:ftrl}
  \DontPrintSemicolon
  \SetKwInOut{Input}{Input}
  \SetKwInOut{Output}{Output}

  \Input{ Time horizon $T$, step size $\eta$, $\alpha$-strongly-convex regularizer $R : \calX \to \R$.}

  Initialize history $h_0 = 0$. \\

  \For{$t = 1, 2, \dots, T$}{
    Learner chooses $x_t \in \argmin_{x \in \calX} \sum_{s=1}^{t-1} \circfn{f}_s(x) + \frac{R(x)}{\eta}$. \\
    Set $h_t = A h_{t-1} + B x_t$. \\
    Learner suffers loss $f_t(h_t)$ and observes $f_t$.
  }
\end{algorithm}

\begin{restatable}{theorem}{thmregretupperbound}\label{thm:regret_upper_bound}
  Consider an online convex optimization with unbounded memory problem specified
  by $(\calX, \calH, A, B)$. Let the regularizer $R : \calX \to \R$ be
  $\alpha$-strongly-convex and satisfy $|R(x) - R(\tilde{x})| \leq D$ for all
  $x, \tilde{x} \in \calX$.~\cref{alg:ftrl} with step-size $\eta$ satisfies
  $R_T(\texttt{FTRL}) \leq \frac{D}{\eta} + \eta \frac{T \circfn{L}^2}{\alpha} +
  \eta \frac{T L \circfn{L} H_1}{\alpha}$.
  If $\eta = \sqrt{\frac{\alpha D}{T \circfn{L} (L H_1 + \circfn{L})}}$, then
  \begin{equation*}
    R_T(\texttt{FTRL}) \leq O \left( \sqrt{\frac{D}{\alpha} T L \circfn{L} H_1} \right).
  \end{equation*}
  When $(\calX, \calH, A, B)$ follows linear sequence dynamics with the
  $\xi$-weighted $p$-norm, then all of the above hold with $H_p$ instead of
  $H_1$.
\end{restatable}

The proof of this theorem involves writing the regret as $\sum_t f_t(h_t) -
\circfn{f}_t(x_t)$ + $\sum_t \circfn{f}_t(x_t) - \circfn{f}_t(x^*)$, and
bounding the first term using the $p$-effective memory capacity and the second
term using FTRL. We defer the full proof
to~\cref{sec:appendix_regret_upper_bounds}. The following lower bound shows that
this is tight in the worst-case.

\begin{restatable}{theorem}{thmregretlowerbound}\label{thm:regret_lower_bound}
  There exists an instance of the online convex optimization with unbounded
  memory problem, $(\calX, \calH, A, B)$, that follows linear sequence dynamics
  with the $\xi$-weighted $p$-norm and there exist $L$-Lipschitz continuous loss
  functions $\{ f_t : \calH \to \R \}_{t=1}^T$ such that the regret of any
  algorithm $\calA$ satisfies
  \begin{equation*}
    R_T(\calA) \geq \Omega \left( \sqrt{T L \circfn{L} H_p}\right).
  \end{equation*}
\end{restatable}

The proof of this theorem follows from our lower bound for the special case of
OCO with finite memory~(\cref{thm:regret_lower_bound_finite}), which we discuss in detail
below. However, as we show in~\cref{sec:appendix_regret_lower_bounds}, the lower bound
holds for a much broader class of problems.

\paragraph{Specialization to OCO with finite memory.}
Now we show how our bounds specialize to the special case of OCO with finite
memory~(\cref{subsec:framework_special_cases}). Due to space constraints, we
defer the specialization of the upper and lower bounds for OCO with
$\rho$-discounted infinite memory~(\cref{subsec:framework_special_cases})
to~\cref{sec:appendix_regret_upper_bounds}
and~\cref{sec:appendix_regret_lower_bounds} respectively.

\begin{restatable}{theorem}{thmregretupperboundfinite}\label{thm:regret_upper_bound_finite}
  Consider an online convex optimization with finite memory problem with
  constant memory length $m$ specified by $(\calX, \calH = \calX^m,
  A_{\text{finite},m}, B_{\text{finite},m})$. Let the regularizer $R : \calX \to
  \R$ be $\alpha$-strongly-convex and satisfy $|R(x) - R(\tilde{x})| \leq D$ for
  all $x, \tilde{x} \in \calX$.~\cref{alg:ftrl} with step-size $\eta =
  \sqrt{\frac{\alpha D}{T \circfn{L} (L m^\frac32 + \circfn{L})}}$ satisfies
  \begin{equation*}
    R_T(\texttt{FTRL}) \leq O \left( \sqrt{\frac{D}{\alpha} T L \circfn{L} m^{\frac{3}{2}}} \right) \leq O \left( m \sqrt{\frac{D}{\alpha} T L^2} \right).
  \end{equation*}
\end{restatable}

This follows from using the definition of $A_{\text{finite},m}$ to bound
$\circfn{L}$ and $H_2$ in~\cref{thm:regret_upper_bound}. This improves existing
results~\citep{anavaHM2015online} by a factor of $m^{\nicefrac{1}{4}}$. Our
bound depends on the Lipschitz continuity constants as $\sqrt{L \circfn{L}}$
whereas existing bounds depend as $\circfn{L}$, and $\circfn{L}$ can be as large
as $\sqrt{m} L$~(\cref{thm:lcirc}). We defer the full proof
to~\cref{sec:appendix_regret_upper_bounds} and a detailed comparison with
existing results to~\cref{subsec:appendix_oco_finite_existing}.
\footnote{A similar bound was independently obtained by~\citet[Theorem 20, Appendix B.8]{zhaoWZ2022nonstationary}.}
The following lower bound shows that this is tight in the worst-case.

\begin{restatable}{theorem}{thmregretlowerboundfinite}\label{thm:regret_lower_bound_finite}
  There exists an instance of the online convex optimization with finite memory
  problem with constant memory length $m$, $(\calX, \calH = \calX^m,
  A_{\text{finite},m}, B_{\text{finite},m})$, and there exist $L$-Lipschitz
  continuous loss functions $\{ f_t : \calH \to \R \}_{t=1}^T$ such that the
  regret of any algorithm $\calA$ satisfies
  \begin{equation*}
    R_T(\calA) \geq \Omega \left( m \sqrt{T L^2} \right).
  \end{equation*}
\end{restatable}

To the best of our knowledge, this is the first non-trivial lower bound for OCO
with finite memory with an explicit dependence on the memory length $m$.
\citet{zhaoWZ2022nonstationary} show a lower bound on an OCO with finite memory problem with \emph{memory size 2}. They do so by constructing loss functions whose \emph{Lipschitz constant depends on $m$}.
~\cref{thm:regret_lower_bound_finite} is the first lower bound for OCO with finite memory with \emph{memory size $m$ for general $m$} and \emph{Lipschitz constant independent of $m$}.
Our construction involves three main steps, the first two of which are loosely
based on~\citet{altschulerT2018online}. First, divide time into $N =
\nicefrac{T}{m}$ blocks of size $m$. (For simplicity, assume $T$ is a multiple
of $m$.) Second, sample a random sign $\eps_n$ for each block $n \in [N]$.
Third, for $t > m$ choose 
\begin{equation*}
  f_t(h_t) = \underbrace{\eps_{\ceil{\frac{t}{m}}}}_{(a)} \  \underbrace{L m^{-\frac12}}_{(b)} \  \underbrace{\left( x_{t-m+1} + \dots + x_{m \floor{\frac{t}{m}} + 1} \right)}_{(c)},
\end{equation*}
where term (a) is the random sign $\eps_n$ sampled for the block $n =
\ceil{\nicefrac{t}{m}}$ that $t$ belongs to, term (b) is a scaling factor chosen
while respecting the Lipschitz continuity constraint, and term (c) is a sum over
a \emph{subset} of past decisions. Two important features of this construction
are: (i) a random sign is sampled for each block rather than each round; and
(ii) the loss in round $t$ depends on the history of decisions until and
including the first round of the block that $t$ belongs to. These exploit the
fact that an algorithm cannot overwrite its history and penalize it for its past
decisions even after it observes the random sign $\eps_n$ for the current block.
(See~\cref{fig:lower_bound_oco_finite} for an illustration.) Existing lower
bound proofs for OCO sample a random sign in each round and choose $f_t(x_t)
\propto \eps_t x_t$. A first attempt at extending this for the OCO with finite
memory setting would be to choose $f_t(h_t) \propto \eps_t \sum_{k=0}^{m-1}
x_{t-k}$. However, in constrast to our approach, this does not exploit the fact
that an algorithm cannot overwrite its history and does not suffice for
obtaining a matching lower bound.



\begin{figure}[htb]
  \centering
  \begin{tikzpicture}
    
    \draw[thick, -Triangle] (0,0) -- (13cm,0) node[font=\scriptsize,below left=3pt and -8pt]{Time};
    
    \foreach \x in {0,1,...,11}
    \draw (\x cm, 3pt) -- (\x cm, -3pt);

    \foreach \x in {0,3,...,11}
    \draw[very thick] (\x cm, 3pt) -- (\x cm, -3pt);

    \foreach \x/\descr in {0/x_1,1/x_2,2/x_3,3/x_4,4/x_5,5/x_6,6/x_7,7/x_8,8/x_9,9/x_{10},10/x_{11},11/x_{12}}
    \node[font=\scriptsize, text height=1.75ex, text depth=.5ex] at (\x+.5, -.3) {$\descr$};

    \foreach \x/\descr in {0/\eps_1,3/\eps_2,6/\eps_3,9/\eps_4}
    \node[font=\scriptsize, text height=1.75ex, text depth=.5ex] at (\x+.5, .3) {$\descr$};
    
    \foreach \x in {1,2,3}
    \filldraw[draw=black,fill=SeaGreen] (\x,-.7) rectangle (\x+1, -.9);
    \node[font=\scriptsize, text height=1.75ex, text depth=.5ex] at (6, -.8) {$f_4(h_4) = \eps_2 \  3^{-\frac12} \  (x_2 + x_3 + x_4)$};

    \foreach \x in {2,3}
    \filldraw[draw=black,fill=SeaGreen] (\x,-1.1) rectangle (\x+1, -1.3);
    \foreach \x in {4}
    \filldraw[draw=black,fill=Orange] (\x,-1.1) rectangle (\x+1, -1.3);
    \node[font=\scriptsize, text height=1.75ex, text depth=.5ex] at (6.7, -1.2) {$f_5(h_5) = \eps_2 \  3^{-\frac12} \  (x_3 + x_4)$};

    \foreach \x in {3}
    \filldraw[draw=black,fill=SeaGreen] (\x,-1.5) rectangle (\x+1, -1.7);
    \foreach \x in {4,5}
    \filldraw[draw=black,fill=Orange] (\x,-1.5) rectangle (\x+1, -1.7);
    \node[font=\scriptsize, text height=1.75ex, text depth=.5ex] at (7.4, -1.6) {$f_6(h_6) = \eps_2 \  3^{-\frac12} \  (x_4)$};

    \node[font=\scriptsize, text height=1.75ex, text depth=.5ex] at (1, 1) {$T = 12, m = 3, L = 1$.};
    \node[font=\scriptsize, text height=1.75ex, text depth=.5ex] at (5.5, 1) {Resample every $m$ rounds.};
    \draw (3.5, 1) -- (4, 1);
    \draw (3.5, .6) -- (3.5, 1);
    \draw[-Triangle] (3.5, .6) -- (3.5, 0.5);
    
  \end{tikzpicture}
  \caption{An illustration of the loss functions $f_t$ for the OCO with finite memory lower bound.}
  \label{fig:lower_bound_oco_finite}
\end{figure}
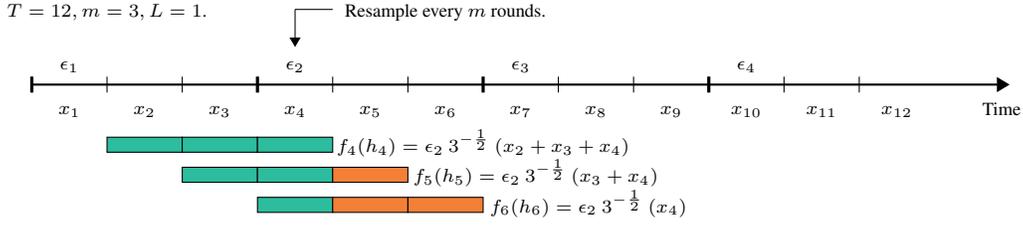

\paragraph{Comparison of upper bound with prior work.}
The algorithmic ideas and analysis for our regret upper bound are influenced
by~\citet{anavaHM2015online}. However, an important innovation in our work is
the use of weighted norms in the case of linear sequence dynamics. This is a
simple but powerful way of encoding prior knowledge about a problem, and allows
us to derive non-trivial regret bounds in the case of unbounded-length
histories. The technical complications that arise are captured in bounding the
relevant quantities of interest, e.g., the Lipschitz constant $\circfn{L}$, the
operator norm $\| A^k \|$, etc. Furthermore, using weighted norms even leads to
improved regret bounds for some applications. Indeed, consider the application
to online linear control with adversarial
disturbances~(\cref{subsec:applications_olc}). Our framework and upper bound
applied to this problem~(\cref{thm:regret_upper_bound_olc}) improve upon the
existing upper bound, which used a finite memory approximation.
See~\cref{lemma:appendix_olc_Ak,lemma:appendix_olc_L} for an illustration of the
technical details involved when using weighted norms.


\section{Applications}
\label{sec:applications}


In this section we apply our framework to online linear
control~(\cref{subsec:applications_olc}) and online performative
prediction~(\cref{subsec:applications_opp}). We defer expanded details and
proofs to~\cref{sec:appendix_olc,sec:appendix_opp} respectively.


\subsection{Online Linear Control}
\label{subsec:applications_olc}


\paragraph{Background.}
Online linear control (OLC) is the problem of controlling a system with linear
dynamics, adversarial disturbances, and adversarial and convex losses. It
combines aspects from control theory and online learning. We refer the reader
to~\citet{agarwalBHKS2019online} for more details. Here, we introduce the basic
mathematical setup of the problem.

Let $\calS \subseteq \R^{d_s}$ and $\calU \subseteq \R^{d_u}$ denote the state
and control spaces. Let $s_t$ and $u_t$ denote the state and control at time $t$
with $s_0$ being the inital state. The system evolves according to the linear
dynamics $s_{t+1} = F s_t + G u_t + w_t$, where $F \in \R^{d_s \times d_s}, G
\in \R^{d_s \times d_u}$ are matrices satisfying $\| F \|_2, \| G \|_2 \leq
\kappa$ and $w_t \in \R^{d_s}$ is an adversarially chosen disturbance with $\| w
\|_2 \leq W$. Without loss of generality, we assume that $\kappa, W \geq 1$,
$d_s = d_u = d$, and also define $w_{-1} = s_0$. For $t = 0, \dots, T-1$, let
$c_t : \calS \times \calU \to [0,1]$ be convex loss functions chosen by an
oblivous adversary. The functions $c_t$ satisfy the following Lipschitz
condition: if $\| s \|_2, \| u \|_2 \leq D_\calX$, then $\| \nabla_s c_t(s, u)
\|, \| \nabla_u c_t(s, u) \| \leq L_0 D_\calX$. The goal in online linear
control is to choose a sequence of policies that yield a sequence of controls
$u_t$ to minimize the regret $ R_T(\Pi) = \sum_{t=0}^{T-1} c_t(s_t, u_t) -
\min_{\pi^* \in \Pi} \sum_{t=0}^{T-1} c_t(s_t^{\pi^*}, u_t^{\pi^*}), $ where
$s_t$ evolves according to linear dynamics stated above, $\Pi$ denotes a
controller class, and $s_t^{\pi^*}, c_t^{\pi^*}$ denote the state and control at
time $t$ when the controls are chosen according to $\pi^*$.

A very simple controller class is constant input, i.e., $\Pi = \{ \pi_u : \pi(s)
= u \in \calU \}$. In this case, the history $h_t$ can be represented by the
finite-dimensional state $s_t$, and the operators can be set to $A = F$ and $B =
\nicefrac{G}{\| G \|}$.
However, like previous work~\citep{agarwalBHKS2019online} we focus on the class
of $(\kappa, \rho)$-strongly stable linear controllers, $\calK$, where $K \in
\calK$ satisfies $F - GK = H L H^{-1}$ with $\| K \|_2, \| H \|_2, \| H^{-1}
\|_2 \leq \kappa$ and $\| L \|_2 = \rho < 1$. Given such a controller, the
inputs are chosen as linear functions of the current state, i.e., $u_t = - K
s_t$. Unfortunately, parameterizing $u_t$ directly with a linear controller as
$u_t = -K s_t$ leads to a non-convex problem because $s_t$ is a non-linear
function of $K$, e.g., if disturbances are $0$, then $s_t = (F-GK)^t s_0$. An
alternative parameterization is the disturbance-action controller (DAC).

\begin{definition}\label{def:disturbance_action_controller}
  Let $K \in \calK$ be fixed. The class of disturbance-action controllers (DACs) $\calM_K$ is $\{ (K, M) : M = (M^{[s]})_{s=0}^\infty \}$, where $M^{[s]} \in \R^{d \times d}$ satisfies $\| M^{[s]} \|_2 \leq \kappa^4 \rho^s$. The control in round $k$ is chosen as $u_k = -K s_k + \sum_{s=1}^{k+1} M^{[s]} w_{k-s}$.
\end{definition}
The class of such DACs has two important properties. First, it acts on the
entire history of past disturbances. Consequently, given an arbitrary $K \in
\calK$, every $K^* \in \calK$ can be expressed as a DAC $(K, M) \in \calM_K$
with $M = (M^{[1]}, \dots, M^{[T+1]}, 0, \dots)$~\citep[Section
16.5]{agarwalJKS2019reinforcement}. That is, $\calK \subseteq \calM_K$ and it
suffices to compute regret against $\calM_K$ instead of $\calK$. For the rest of
this paper we fix $K \in \calK$ and denote $\wt{F} = F - GK$. Second, suppose
$M_t = (M_t^{[s]})_{s=0}^\infty$ is the parameter chosen in round $t$ and the
control $u_t$ is chosen according to the DAC $(K, M_t)$. Then, $s_t$ and $u_t$
are \emph{linear} functions of the parameters, which implies that $c_t$ is
convex in the parameters. (See the next paragraph on ``Formulation as OCO with
Unbounded Memory'' for a formula.) A similar parameterization was first
considered for online linear control by~\citet{agarwalBHKS2019online} and is
based on similar ideas in control theory, e.g.,
Youla~\citep{youlaJB1976modern,kuvcera1975stability} and
SLS~\citep{wangMD2019system,andersonDLM2019system}.


\paragraph{Formulation as OCO with Unbounded Memory.}
The first step is a change of variables with respect to the control inputs from
linear controllers to DACs and the second is a corresponding change of variables
for the state.
Define the decision space $\calX = \{ M = (M^{[s]}) : M^{[s]} \in \R^{d \times
d}, \| M^{[s]} \|_2 \leq \kappa^4 \rho^s \}$. Define the history space $\calH$
to be the set consisting of sequences $h = (Y_k)$, where $Y_0 \in \calX$ and
$Y_k = \wt{F}^{k-1} G X_k$ for $X_k \in \calX, k \geq 1$. (Recall $\wt{F} = F -
GK$.) Define weighted norms $\| M \|_\calX^2 = \sum_{s=1}^\infty \rho^{-s} \|
M^{[s]} \|_F^2$ and $\| h \|_\calH^2 = \sum_{k=0}^\infty \xi_k^2 \| Y_k
\|_\calX^2$,
where the weights $(\xi_k)$ are defined as $\xi = (1, 1, 1, \rho^{-\frac12},
\rho^{-1}, \rho^{-\frac32}, \dots)$. Define the linear operators $A : \calH \to
\calH$ and $B : \calW \to \calH$ as $A((Y_0, Y_1, \dots)) = (0, G Y_0, \wt{F}
Y_1, \wt{F} Y_2, \dots)$ and $B(M) = (M, 0, 0, \dots)$. Note that the problem
follows linear sequence dynamics with the $\xi$-weighted
$2$-norm~(\cref{def:sequence}). The weights in the norms on $\calX$ and $\calH$
increase exponentially. However, the norms $\| M^{[s]} \|_F^2$ and $\|
\wt{F}^{k-1} G \|_F^2$ decrease exponentially as well: by definition of $\calX$
and the assumption on $\wt{F} = F - GK$ for $K \in \calK$. Leveraging these
exponential decays to define exponentially increasing weights is crucial for 
deriving our regret bounds that are stronger than existing results.

Construct the functions $f_t : \calH \to \R$ that correspond to $c_t(s_t, u_t)$
as follows. Given a sequence of decisions $(M_0, \dots, M_t)$, the history at
the end of round $t$ is given by $h_t = (M_t, G M_{t-1}, \wt{F} G M_{t-2},
\dots, \wt{F}^{t-1} G M_0, 0, \dots)$. A simple inductive argument shows that
the state and control in round $t$ can be written as functions of $h_t$ as $s_t
= \wt{F}^t s_0 + \sum_{k=0}^{t-1} \sum_{s=1}^{k+1} \wt{F}^{t-k-1} G M_k^{[s]}
w_{k-s} + w_{t-1}$ and $u_t = -K s_t + \sum_{s=1}^{t+1} M_t^{[s]} w_{t-s}$.
Define the functions $f_t : \calH \to \R$ by $f_t(h) = c_t(s, u)$, where $s$ and
$u$ are the state and control determined by the history as above. Note that
$f_t$ is parameterized by the past disturbances. Since the state and control are
linear functions of the history and $c_t$ is convex, this implies that $f_t$ is
convex. Now, given the above formulation and the fact that the class of
disturbance-action controllers is a superset of the class of
$(\kappa,\rho)$-strongly-stable linear controllers, we have that the policy
regret for the online linear control problem is at most the policy regret,
$\sum_{t=0}^{T-1} f_t(h_t) - \min_{M \in \calX} \sum_{t=0}^{T-1}
\circfn{f}_t(M)$. The following is our main result for online linear control and
it improves existing results~\citep{agarwalBHKS2019online} by a factor of $O( d
\log(T)^{3.5} \kappa^5 (1-\rho)^{-1})$. See~\cref{subsec:appendix_olc_existing}
for a detailed comparison.

\begin{restatable}{theorem}{thmregretupperboundolc}\label{thm:regret_upper_bound_olc}
  Consider the online linear control problem as defined
  in~\cref{subsec:applications_olc}. Suppose the decisions in round $t$ are
  chosen using~\cref{alg:ftrl}. Then, the upper bound on the policy regret is
  \begin{equation}
    O \left( L_0 W^2 \sqrt{T} d^\frac12 \kappa^{17} (1-\rho)^{-4.5} \right).
  \end{equation}
\end{restatable}


\paragraph{Comparison with prior and concurrent work.}
Existing works solve OLC (and its extensions) by making multiple finite memory
approximations. First, they formulate the problem as OCO with finite memory.
This requires bounding numerous error terms because the problem is inherently an
OCO with unbounded memory problem. We bypass these error analysis steps entirely
because the problem fits into our framework naturally. Second, existing works
use the parameterization from~\citet{agarwalBHKS2019online} that only acts on a
fixed, constant number of past disturbances. In particular, existing works use a
``truncated'' DAC policy that is a sequence of $d \times d$ matrices of length
$2 \kappa^4 (1-\rho)^{-1} \log T$. Our DAC policy acts on the entire history of
disturbances and is a sequence of $d \times d$ matrices of unbounded length.
Yet, we capture the dimension of this infinite-dimensional space in a way that
still improves the overall bound, including completely eliminating the
dependence on $\log T$, and improving the dependence on $d, \kappa$, and
$(1-\rho)$. This improvement comes from our novel use of weighted norms on
the history and decision spaces. These norms allow us to give tighter bounds on the relevant quantities
in the regret upper bound, e.g., $\| A^k \|$~(\cref{lemma:appendix_olc_Ak}) and
$\circfn{L}$~(\cref{lemma:appendix_olc_L}).

In complementary concurrent work,~\citet{lin2022online} focus on a more general
online control problem. They improve regret bounds for this general version by a
factor of $\log T$ compared to existing reductions to OCO with finite memory.
They do so by using that the impact of a past policy decays geometrically with
time. On the other hand, the primary focus of our work is studying the complete
dependence of present losses on the entire history in OCO. Applying our
resulting OCO with unbounded memory framework to OLC, we improve upon existing
results for OLC by removing all $\log T$ factors and improving the dependence on
$d, \kappa$, and $(1-\rho)$.



\subsection{Online Performative Prediction}
\label{subsec:applications_opp}


\paragraph{Background.}
In many applications of machine learning the algorithm's decisions influence the
data distribution, e.g., online labor
markets~\citep{anagnostopoulosCFLT2018algorithms,horton2010online}, predictive
policing~\citep{lumI2016predict}, on-street
parking~\citep{dowlingRZ2020modeling,pierceS2018sfpark}, vehicle sharing
markets~\citep{banerjeeRJ2015pricing}, etc. Motivated by such applications,
several works have studied the problem of performative prediction, which models
the data distribution as a function of the decision-maker's
decision~\citep{perdomoZMH2020performative,mendler-dunnerPZH2020stochastic,millerPZ2021outside,brownHK2022performative,rayRDF2022decision,jagadeesanZM2022regret}.
Most of these works view the problem as a stochastic optimization
problem;~\citet{jagadeesanZM2022regret} adopt a regret minimization perspective.
We refer the reader to these citations for more details. As a natural extension
to existing works, we introduce an online learning variant of performative
prediction with geometric decay~\citep{rayRDF2022decision} that differs from the
original formulations in a few key ways.

Let the decision set $\calX \subseteq \R^d$ be closed and convex with $\| x \|_2
\leq D_\calX$. Let $p_1$ denote the initial data distribution over the instance
space $\calZ$. In each round $t \in [T]$, the learner chooses a decision $x_t
\in \calX$ and an oblivious adversary chooses a loss function $l_t : \calX
\times \calZ \to [0,1]$, and then the learner suffers the loss $L_t(x_t) = \E_{z
\sim p_t} \left[ l_t(x_t, z) \right]$, where $p_t = p_t(x_1, \dots, x_t)$ is the
data distribution in round $t$. The goal in our online learning setting is to
minimize the difference between the algorithm's total loss and the total loss of
the best fixed decision, $\sum_{t=1}^T \E_{z \sim p_t} \left[ l_t(x_t, z)
\right] - \min_{x \in \calX} \sum_{t=1}^T \E_{z \sim p_t(x)} \left[ l_t(x, z)
\right]$, where $p_t(x) = p_t(x, \dots, x)$ is the data distribution in round
$t$ had $x$ been chosen in all rounds so far. This measure is similar to
performative regret~\citep{jagadeesanZM2022regret} and is a natural
generalization of performative optimality~\citep{perdomoZMH2020performative} for
an online learning formulation.

We make the following assumptions. First, the loss functions $l_t$ are convex
and $L_0$-Lipschitz continuous. Second, the data distribution satisfies for all
$t \geq 1$, $p_{t+1} = \rho p_t + (1-\rho) \calD(x_t)$, where $\rho \in (0,1)$
and $\calD(x_t)$ is a distribution over $\calZ$ that depends on the decision
$x_t$~\citep{rayRDF2022decision}. Third, $\calD(x)$ is a location-scale
distribution: $z \sim \calD(x)$ iff $z \sim \xi + F x$, where $F \in \R^{d
\times d}$ satisfies $\| F \|_2 < \infty$ and $\xi$ is a random variable with
mean $\mu$ and covariance $\Sigma$~\citep{rayRDF2022decision}.

Our problem formulation differs from existing work in the following ways. First,
we adopt an online learning perspective on performative prediction with
geometric decay, whereas~\citet{rayRDF2022decision} adopt a stochastic
optimization one. So, we assume that the loss functions $l_t$ are adversarially
chosen, whereas~\citet{rayRDF2022decision} assume $l_t = l$ are fixed. Second,
we assume that the dynamics ($\calD$ and $\rho$) are
known~(Assumption~\ref{ass:feedback}), whereas~\citet{rayRDF2022decision} assume
they are unknown and use samples from the data distribution. We believe that an
appropriate extension of our framework that can deal with unknown linear
operators $A$ and $B$ can be applied to this more difficult setting, and we
leave this as future work. Third, even though~\citet{jagadeesanZM2022regret}
also study an online learning variant of performative prediction, they assume
$l_t = l$ are fixed and the data distribution depends only on the current
decisions, whereas we assume the data distribution depends on the entire history
of decisions.


\paragraph{Formulation as OCO with Unbounded Memory.}
Let $\rho \in (0, 1)$. Let the decision space $\calX \subseteq \R^d$ be closed
and convex with the $2$-norm. Let the history space $\calH$ be the
$\ell^1$-direct sum of countably infinite number of copies of $\calX$. Define the
linear operators $A : \calH \to \calH$ and $B : \calW \to \calH$ as $A((y_0,
y_1, \dots)) = (0, \rho y_0, \rho y_1, \dots)$ and $B(x) = (x, 0, \dots)$. The
problem is an OCO with $\rho$-discounted infinite memory problem and follows
linear sequence dynamics with the $1$-norm~(\cref{def:sequence}). Given a
sequence of decisions $(x_k)_{k=1}^t$, the history is $h_t = (x_t, \rho x_{t-1},
\dots, \rho^{t-1} x_1, 0, \dots)$ and the data distribution $p_t = p_t(h_t)$
satisfies: $z \sim p_t$ iff $z \sim \sum_{k=1}^{t-1} (1-\rho) \rho^{k-1} (\xi +
F x_{t-k}) + \rho^t p_1$. This follows from the recursive definition of $p_t$
and parametric assumption about $\calD(x)$. Define the functions $f_t : \calH
\to [0,1]$ by $f_t(h_t) = \E_{z \sim p_t} [ l_t(x_t, z) ]$. Now, the original
goal of minimizing the difference between the algorithm's total loss and the
total loss of the best fixed decision is equivalent to minimizing the policy
regret. The following is our main result for online performative prediction.

\begin{restatable}{theorem}{thmregretupperboundopp}\label{thm:regret_upper_bound_opp}
  Consider the online performative prediction problem as defined
  in~\cref{subsec:applications_opp}. Suppose the decisions in round $t$ are
  chosen using~\cref{alg:ftrl}. Then, the upper bound on the policy regret is
  \begin{equation*}
    O \left( D_\calX L_0 \sqrt{T} \| F \|_2 (1-\rho)^{-\frac12} \rho^{-1} \right).
  \end{equation*}
\end{restatable}



\section{Conclusion}
\label{sec:conclusion}


In this paper we introduced a generalization of the OCO framework, ``Online
Convex Optimization with Unbounded Memory'', that allows the loss in the current
round to depend on the entire history of decisions until that point. We proved
matching upper and lower bounds on the policy regret in terms of the time
horizon, the $p$-effective memory capacity (a quantitative measure of the
influence of past decisions on present losses), and other problem
parameters~(\cref{thm:regret_upper_bound,thm:regret_lower_bound}). As a special
case, we proved the first non-trivial lower bound for OCO with finite
memory~(\cref{thm:regret_lower_bound_finite}), which could be of independent
interest, and also improved existing upper
bounds~(\cref{thm:regret_upper_bound_finite}). We illustrated the power of our
framework by bringing together the regret analysis of two seemingly disparate
problems under the same umbrella: online linear
control~(\cref{thm:regret_upper_bound_olc}), where we improve and simplify
existing regret bounds, and online performative
prediction~(\cref{thm:regret_upper_bound_opp}).

There are a number of directions for future research. A natural follow-up is to
consider unknown dynamics (i.e., when the learner does not know the operators
$A$ and $B$) and/or the case of bandit feedback (i.e., when the learner only
observes $f_t(h_t)$). The extension to bandit feedback has been considered in
the OCO and OCO with finite memory literature
~\citep{hazanL2016optimal,bubeckEL2021kernel,zhaoWZZ2021bandit,graduHH2020nonstochastic,casselK2020bandit}.
It is tempting to think about a version where the history is a \emph{nonlinear},
but decaying, function of the past decisions. The obvious challenge is that the
nonlinearity would lead to non-convex losses. It is unclear how to deal with
such issues, e.g., restricted classes of nonlinearities for which the OCO with
unbounded memory perspective is still relevant~\citep{zhangYJZ2014online},
different problem formulations such as online non-convex
learning~\citep{gaoLZ2018online,suggalaN2020online}, etc.

There is a growing body of work on online linear control and its variants that
rely on OCO with finite
memory~\citep{hazanKS2020nonstochastic,agarwalHS2019logarithmic,fosterS2020logarithmic,casselK2020bandit,graduHH2020nonstochastic,liDL2021online,minasyanGSH2021online}.
In this paper we showed how our framework can be used to improve and simplify
regret bounds for the online linear control problem. Another direction for
future work is to use our framework, perhaps with suitable extensions outlined
above, to derive similar improvements for these other variants of online linear
control.



\begin{ack}
  We thank Sloan Nietert and Victor Sanches Portella for helpful discussions. We
  thank Wei-Yu Chan for pointing out an error in the proof
  of~\cref{lemma:dist_actual_ideal_history}. We also thank anonymous AISTATS and
  NeurIPS reviewers for their helpful comments on an older and the current
  version of the paper respectively. This research was partially supported by
  the NSERC Postgraduate Scholarships-Doctoral Fellowship 545847-2020, NSF
  awards CCF-2312774 and OAC-2311521, and a gift from Wayfair.
\end{ack}

\bibliographystyle{plainnat}
\bibliography{refs}


\newpage
\appendix

\section*{Organization of the Appendix}
The appendix is organized as follows:
\begin{itemize}
  \item \cref{sec:appendix_framework} contains the proofs of the results
  from~\cref{sec:framework}.
  \item \cref{sec:appendix_standard_ftrl} contains the proofs of existing
  results about FTRL for completeness.
  \item \cref{sec:appendix_regret_upper_bounds} contains the proofs of upper
  bounds on regret from~\cref{sec:regret_analysis}.
  \item \cref{sec:appendix_regret_lower_bounds} contains the proofs of lower
  bounds on regret from~\cref{sec:regret_analysis}.
  \item \cref{sec:appendix_olc} contains the proofs of results about online
  linear control from~\cref{subsec:applications_olc}.
  \item \cref{sec:appendix_opp} contains the proofs of results about online performative prediction from~\cref{subsec:applications_opp}.
  \item \cref{sec:appendix_implementation} discusses how to implement~\cref{alg:ftrl} efficiently.
  \item \cref{sec:appendix_minibatchftrl} presents an algorithm for OCO with
  unbounded memory that provides the same upper bound on policy regret
  as~\cref{alg:ftrl} while guaranteeting a small number of
  switches~(\cref{alg:minibatch_ftrl}).
  \item \cref{sec:appendix_experiments} presents simulation experiments.
\end{itemize}

\section{Framework}
\label{sec:appendix_framework}


In this section prove~\cref{thm:lcirc}. But first we prove a lemma that we use
for proofs involving linear sequence dynamics with the $\xi$-weighted
$p$-norm~(\cref{def:sequence}). Recall that $\| \cdot \|_\calU$ denotes the norm
associated with a space $\calU$ and the operator norm $\| L \|$ for a linear
operator $L : \calU \to \calV$ is defined as $\| L \| = \max_{u : \| u \|_\calU
\leq 1} \| L u \|_\calV$.

\begin{lemma}\label{lemma:sequence_norm_inequality}
  Consider an online convex optimization with unbounded memory problem specified
  by $(\calX, \calH, A, B)$. If $(\calX, \calH, A, B)$ follows linear sequence
  dynamics with the $\xi$-weighted $p$-norm for $p \geq 1$, then for all $k \geq
  1$
  \begin{equation*}
    \xi_k \| A_{k-1} \cdots A_0 \| \leq \| A^k \|.
  \end{equation*}
\end{lemma}

\begin{proof}
  Let $x \in \calX$ with $\| x \|_\calX = 1$. We have
  \begin{equation*}
    \xi_k \| A_{k-1} \cdots A_0 x \|_\calX
    = \| A^k (x, 0, \dots) \|_\calH
    \leq \| A^k \| \| (x, 0, \dots) \|_\calH
    \leq \| A^k \|,
  \end{equation*}
  where the last inequality follows because $\| (x, 0, \dots) \|_\calH = \xi_0
  \| x \|_\calX$ and $\xi_0 = 1$ by~\cref{def:sequence}. Therefore, $\| A_{k-1}
  \cdots A_0 \| \leq \| A^k \|$.
\end{proof}

\thmlcirc*

\begin{proof}
  Let $x, \tilde{x} \in \calX$. For the general case, we have
  \begin{align*}
    \left\vert \circfn{f}_t(x) - \circfn{f}_t(\tilde{x}) \right\vert
    &= \left\vert f_t \left( \sum_{k=0}^{t-1} A^k B x \right) - f_t \left( \sum_{k=0}^{t-1} A^k B \tilde{x} \right) \right\vert & \text{ by~\cref{def:circfn}} \\
    &\leq L \  \left\| \sum_{k=0}^{t-1} A^k B (x - \tilde{x}) \right\|_\calH & f_t \text{ is } L\text{-Lipschitz continuous} \\
    &\leq L \  \sum_{k=0}^{t-1} \| A^k \| \| B \| \| x - \tilde{x} \|_\calX \\
    &\leq L \  \sum_{k=0}^{t-1} \| A^k \| \| x - \tilde{x} \|_\calX & \text{ by Assumption }~\ref{ass:norm_B} \\
    &\leq L \  \sum_{k=0}^{\infty} \| A^k \| \| x - \tilde{x} \|_\calX.
  \end{align*}
  
  If $(\calH, \calX, A, B)$ follows linear sequence dynamics with the
  $\xi$-weighted $p$-norm for $p \geq 1$, then we have
  \begin{align*}
    \left\vert \circfn{f}_t(x) - \circfn{f}_t(\tilde{x}) \right\vert
    &= \left\vert f_t \left( \sum_{k=0}^{t-1} A^k B x \right) - f_t \left( \sum_{k=0}^{t-1} A^k B \tilde{x} \right) \right\vert & \text{ by~\cref{def:circfn}} \\
    &\leq L \  \left\| \sum_{k=0}^{t-1} A^k B (x - \tilde{x}) \right\|_\calH & f_t \text{ is } L\text{-Lipschitz continuous} \\
    &= L \  \left\| (0, A_0 (x - \tilde{x}), A_1 A_0 (x - \tilde{x}), \dots) \right\| & \text{ by~\cref{def:sequence}} \\
    &= L \  \left( \sum_{k=0}^{t-1} \xi_k^p \left\| A_{k-1} \cdots A_0 (x - \tilde{x}) \right\|^p \right)^{\frac{1}{p}} & \text{ by~\cref{def:sequence}} \\
    &\leq L\  \left( \sum_{k=0}^{t-1} \| A^k \|^p \right)^{\frac{1}{p}} \left\| x - \tilde{x} \right\|_\calX & \text{ by~\cref{lemma:sequence_norm_inequality}} \\
    &\leq L\  \left( \sum_{k=0}^{\infty} \| A^k \|^p \right)^{\frac{1}{p}} \left\| x - \tilde{x} \right\|_\calX. \qedhere
  \end{align*}
\end{proof}

\section{Standard Analysis of Follow-the-Regularized-Leader}
\label{sec:appendix_standard_ftrl}


In this section we state and prove some existing results about the
follow-the-regularized-leader (FTRL)
algorithm~\citep{shalevshwartzS2006online,abernethyHR2008competing}. These
results are well known in the literature, but we prove them here for
completeness and use them in the remainder of the paper. We use the below
results for functions $\circfn{f}_t$ with Lipschitz constants $\circfn{L}$.
However, in this section we use a more general notation, denoting functions by
$g_t$ and their Lipschitz constant by $L_g$.

Consider the following setup for an online convex optimization (OCO) problem.
Let $T$ denote the time horizon. Let the decision space $\calX$ be a closed,
convex subset of a Hilbert space and $g_t : \calX \to \R$ be loss functions
chosen by an oblivious adversary. The functions $g_t$ are convex and
$L_g$-Lipschitz continuous. The game between the learner and the adversary
proceeds as follows. In each round $t \in [T]$, the learner chooses $x_t \in
\calX$ and the learner suffers loss $g_t(x_t)$. The goal of the learner is to
minimize (static) regret,
\begin{equation}
  \label{eq:oco_regret}
  R_T^{\text{static}} = \sum_{t=1}^T g_t(x_t) - \min_{x \in \calX} \sum_{t=1}^T g_t(x).
\end{equation}
Let $R : \calX \to \R$ be an $\alpha$-strongly convex regularizer satisfying
$|R(x) - R(\tilde{x})| \leq D$ for all $x, \tilde{x} \in \calX$. The FTRL
algorithm chooses iterates $x_t$ as
\begin{equation}
  \label{eq:oco_ftrl}
  x_t \in \argmin_{x \in \calX} \sum_{s=1}^{t-1} g_s(x) + \frac{R(x)}{\eta},
\end{equation}
where $\eta$ is a tunable parameter referred to as the step-size. In what
follows, let $g_0 = \frac{R}{\eta}$. The analysis in this section closely
follows~\citet{karlin2017notes}.

\begin{lemma}\label{lemma:oco_ftrl_lemma_1}
  For all $x \in \calX$, FTRL~(\cref{eq:oco_ftrl}) satisfies
  \begin{equation*}
    \sum_{t=0}^T g_t(x) \geq \sum_{t=0}^T g_t(x_{t+1}).
  \end{equation*}
\end{lemma}

\begin{proof}
  We use proof by induction on $T$. The base case is $T = 0$. By definition,
  $x_1 \in \argmin_{x \in \calX} R(x)$. Therefore, $R(x) \geq R(x_1)$ for all $x
  \in \calX$. Recalling the notation $g_0 = \frac{R}{\eta}$ proves the base
  case. Now, assume that the lemma is true for $T-1$. That is,
  \begin{equation*}
    \sum_{t=0}^{T-1} g_t(x) \geq \sum_{t=0}^{T-1} g_t(x_{t+1}).
  \end{equation*}
  Let $x \in \calX$ be arbitrary. Since $x_{T+1} \in \argmin_{x \in \calX}
  \sum_{t=0}^T g_t(x)$, we have
  \begin{align*}
    \sum_{t=0}^T g_t(x)
    &\geq \sum_{t=0}^{T} g_t(x_{T+1}) \\
    &= \sum_{t=0}^{T-1} g_t(x_{T+1}) + g_T(x_{T+1}) \\
    &\geq \sum_{t=0}^{T-1} g_t(x_{t+1}) + g_T(x_{T+1}) & \text{ by inductive hypothesis} \\
    &= \sum_{t=0}^T g_t(x_{t+1}).
  \end{align*}
  This completes the proof.
\end{proof}

\begin{lemma}\label{lemma:oco_ftrl_lemma_2}
  For all $x \in \calX$, FTRL~(\cref{eq:oco_ftrl}) satisfies
  \begin{equation*}
    \sum_{t=1}^T g_t(x_t) - \sum_{t=1}^{T} g_t(x) \leq \frac{D}{\eta} + \sum_{t=1}^T g_t(x_t) - g_t(x_{t+1}).
  \end{equation*}
\end{lemma}

\begin{proof}
  Note that
  \begin{equation*}
    \sum_{t=1}^T g_t(x_t) - \sum_{t=1}^T g_t(x) \leq \sum_{t=1}^T g_t(x_t) - \sum_{t=1}^T g_t(x) + g_0(x) - g_0(x_1)
  \end{equation*}
  because $x_1 \in \argmin_{x \in \calX} g_0(x)$. The proof of this lemma now
  follows by using the above inequality,~\cref{lemma:oco_ftrl_lemma_1}, the
  definition $g_0 = \frac{R}{\eta}$, and the definition of $D$.
\end{proof}

\begin{theorem}\label{thm:oco_ftrl}
  FTRL~(\cref{eq:oco_ftrl}) satisfies
  \begin{equation*}
    \| x_{t+1} - x_t \|_\calX \leq \eta \frac{L_g}{\alpha}
    \quad
    \text{ and }
    \quad
    R_T^{\text{static}} \leq \frac{D}{\eta} + \eta \frac{T L_g^2}{\alpha}.
  \end{equation*}
  Choosing $\eta = \sqrt{\frac{\alpha D}{T L_g^2}}$ yields
  \begin{equation*}
    R_T^{\text{static}} \leq O \left( \sqrt{ \frac{D}{\alpha} T L_g^2} \right).
  \end{equation*}
\end{theorem}

\begin{proof}
  Let $x^* \in \argmin_{x \in \calX} \sum_{t=1}^T g_t(x)$.
  Using~\cref{lemma:oco_ftrl_lemma_2} we have
  \begin{equation}\label{eq:thm_oco_ftrl_eq_1}
    \sum_{t=1}^T g_t(x_t) - \sum_{t=1}^T g_t(x^*) \leq \frac{D}{\eta} + \sum_{t=1}^T g_t(x_t) - g_t(x_{t+1}).
  \end{equation}
  We can bound the summands in the sum above as follows. Define $G_t(x) =
  \sum_{s=0}^{t-1} g_s(x)$. Then, $x_t \in \argmin_{x \in \calX} G_t(x)$. and
  $x_{t+1} \in \argmin_{x \in \calX} G_{t+1}(x)$. Since $\{g_s\}_{s=1}^T$ are
  convex, $R$ is $\alpha$-strongly-convex, and $g_0 = \frac{R}{\eta}$, we have
  that $G_t$ is $\frac{\alpha}{\eta}$-strongly-convex. So,
  \begin{align*}
    G_t(x_{t+1}) &\geq G_t(x_t) + \frac{\alpha}{2 \eta} \| x_{t+1} - x_t \|_\calX^2, \\
    G_{t+1}(x_t) &\geq G_{t+1}(x_{t+1}) + \frac{\alpha}{2 \eta} \| x_{t+1} - x_t \|_\calX^2.
  \end{align*}
  Adding the above two inequalities yields
  \begin{equation}\label{eq:thm_oco_ftrl_eq_2}
    g_t(x_t) - g_t(x_{t+1}) \geq \frac{\alpha}{\eta} \| x_{t+1} - x_t \|_\calX^2.
  \end{equation}
  Since $g_t$ are convex and $L_g$-Lipschitz continuous, we also have
  \begin{equation}\label{eq:thm_oco_ftrl_eq_3}
    g_t(x_t) - g_t(x_{t+1}) \leq L_g \| x_{t+1} - x_t \|_\calX.
  \end{equation}
  Combining~\cref{eq:thm_oco_ftrl_eq_2,eq:thm_oco_ftrl_eq_3} we have
  \begin{equation*}
    \| x_{t+1} - x_t \|_\calX \leq \eta \frac{L_g}{\alpha}.
  \end{equation*}
  This proves the first part of the theorem. Now, using this
  in~\cref{eq:thm_oco_ftrl_eq_3} we have
  \begin{equation}
    g_t(x_t) - g_t(x_{t+1}) \leq \eta \frac{L_g^2}{\alpha}.
  \end{equation}
  Finally, substituting this in~\cref{eq:thm_oco_ftrl_eq_1} proves the second
  part of the theorem.
\end{proof}

\section{Regret Analysis: Upper Bounds}
\label{sec:appendix_regret_upper_bounds}


First we prove a lemma that bounds the difference in the value of $f_t$
evaluated at the actual history $h_t$ and an idealized history that would have
been obtained by playing $x_t$ in all prior rounds.

\begin{lemma}\label{lemma:dist_actual_ideal_history}
  Consider an online convex optimization with unbounded memory problem specified
  by $(\calX, \calH, A, B)$. If the decisions $(x_t)$ are generated
  by~\cref{alg:ftrl}, then
  \begin{equation*}
    \left\vert f_t(h_t) - \circfn{f}_t(x_t) \right\vert \leq \eta \frac{L \circfn{L} H_1}{\alpha}
  \end{equation*}
  for all rounds $t$. When $(\calX, \calH, A, B)$ follows linear sequence
  dynamics with the $\xi$-weighted $p$-norm for $p \geq 1$, then
  \begin{equation*}
    \left\vert f_t(h_t) - \circfn{f}_t(x_t) \right\vert \leq \eta \frac{L \circfn{L} H_p}{\alpha}
  \end{equation*}
  for all rounds $t$.
\end{lemma}

\begin{proof}
  We have
  \begin{align}
    \left\vert f_t(h_t) - \circfn{f}_t(x_t) \right\vert
    &= \left\vert f_t(h_t) - f_t\left( \sum_{k=0}^{t-1} A^k B x_t \right) \right\vert & \text{ by~\cref{def:circfn}} \nonumber \\
    &\leq L \left\| h_t - \sum_{k=0}^{t-1} A^k B x_t \right\| & \text{ by Assumption~\ref{ass:lipschitz}} \nonumber \\
    &= L \left\| \sum_{k=0}^{t-1} A^k B x_{t-k} - \sum_{k=0}^{t-1} A^k B x_t \right\| & \text{ by definition of } h_t \nonumber \\
    &= L \underbrace{\left\| \sum_{k=0}^{t-1} A^k B (x_{t-k} - x_t) \right\|}_{(a)}. \label{eq:lemma_dist_actual_ideal_history_1}
  \end{align}
  First consider the general case where $(\calX, \calH, A, B)$ does not
  necessarily follow linear sequence dynamics. We can bound the term (a) as
  \begin{align*}
    \left\| \sum_{k=0}^{t-1} A^k B (x_{t-k} - x_t) \right\|
    &\leq \sum_{k=0}^{t-1} \left\| A^k B \right\| \left\| x_t - x_{t-k} \right\| \\
    &\leq \sum_{k=0}^{t-1} \left\| A^k B \right\| k \eta \frac{\circfn{L}}{\alpha} & \text{by~\cref{thm:oco_ftrl}} \\
    &\leq \sum_{k=0}^{t-1} \left\| A^k \right\| k \eta \frac{\circfn{L}}{\alpha} & \text{by Assumption~\ref{ass:norm_B}} \\
    &\leq \eta \frac{\circfn{L}}{\alpha} H_1.
  \end{align*}
  Plugging this into~\cref{eq:lemma_dist_actual_ideal_history_1} completes the
  proof for the general case. Now consider the case when $(\calX, \calH, A, B)$
  follows linear sequence dynamcis with the $\xi$-weighted $p$-norm. We can
  bound the term (a) as
  \begin{align*}
    \left\| \sum_{k=0}^{t-1} A^k B (x_{t-k} - x_t) \right\|
    &= \left\| (0, A_0 (x_t - x_{t-1}), A_1 A_0 (x_t - x_{t-2}), \dots) \right\| & \text{by~\cref{def:sequence}} \\
    &= \left( \sum_{k=0}^{t-1} \xi_k^p \left\| A_{k-1} \cdots A_0 (x_t - x_{t-k}) \right\|^p \right)^{\frac{1}{p}} & \text{by~\cref{def:sequence}} \\
    &\leq \left( \sum_{k=0}^{t-1} \xi_k^p \left\| A_{k-1} \cdots A_0 \right\|^p \left\|x_t - x_{t-k} \right\|^p \right)^{\frac{1}{p}} \\
    &\leq \left( \sum_{k=0}^{t-1} \left\| A^k \right\|^p \left\|x_t - x_{t-k} \right\|^p \right)^{\frac{1}{p}} & \text{ by~\cref{lemma:sequence_norm_inequality}} \\
    &\leq \eta \frac{\circfn{L}}{\alpha} \left( \sum_{k=0}^{t-1} \left\| A^k \right\|^p k^p \right)^{\frac{1}{p}} & \text{by~\cref{thm:oco_ftrl}} \\
    &\leq \eta \frac{\circfn{L}}{\alpha} H_{p}.
  \end{align*}
  Plugging this into~\cref{eq:lemma_dist_actual_ideal_history_1} completes the
  proof.
\end{proof}


Now we restate and prove~\cref{thm:regret_upper_bound}

\thmregretupperbound*

\begin{proof}
  First consider the general case where $(\calX, \calH, A, B)$ does not
  necessarily follow linear sequence dynamics. Let $x^* \in \argmin_{x \in
  \calX} \sum_{t=1}^T \circfn{f}_t(x)$. Note that we can write the regret as
  \begin{align*}
    R_T(\texttt{FTRL})
    &= \sum_{t=1}^T f_t(h_t) - \min_{x \in \calX} \sum_{t=1}^T \circfn{f}_t(x) \\
    &= \underbrace{\sum_{t=1}^T f_t(h_t) - \circfn{f}_t(x_t)}_{(a)} + \underbrace{\sum_{t=1}^T \circfn{f}_t(x_t) - \circfn{f}_t(x^*)}_{(b)}.
  \end{align*}
  We can bound term (a) using~\cref{lemma:dist_actual_ideal_history} and term
  (b) using~\cref{thm:oco_ftrl}. Therefore, we have
  \begin{align*}
    R_T(\texttt{FTRL})
    &= \underbrace{\sum_{t=1}^T f_t(h_t) - \circfn{f}_t(x_t)}_{(a)} + \underbrace{\sum_{t=1}^T \circfn{f}_t(x_t) - \circfn{f}_t(x^*)}_{(b)} \\
    &\leq \eta \frac{T L \circfn{L} H_1}{\alpha} + \frac{D}{\eta} + \eta \frac{T \circfn{L}^2}{\alpha}.
  \end{align*}
  Choosing $\eta = \sqrt{\frac{\alpha D}{T \circfn{L} (L H_1 + \circfn{L})}}$
  yields
  \begin{equation*}
    R_T(\texttt{FTRL}) \leq O \left( \sqrt{\frac{D}{\alpha} T L \circfn{L} H_1} \right),
  \end{equation*}
  where we used the definition of $p$-effective memory capacity~(\cref{def:Hp})
  and the bound on $\circfn{L}$~(\cref{thm:lcirc}) to simplify the above
  expression. This completes the proof for the general case. The proof for when
  $(\calX, \calH, A, B)$ follows linear sequence dynamcis with the
  $\xi$-weighted $p$-norm is the same as above, except we bound the term (a)
  above using ~\cref{lemma:dist_actual_ideal_history} for linear sequence
  dynamics.
\end{proof}


Now we restate and prove~\cref{thm:regret_upper_bound_finite}.

\thmregretupperboundfinite*

The OCO with finite memory problem, as defined in the literature, follows linear
sequence dynamics with the $2$-norm. In this subsection we consider a more
general version of the OCO with finite memory problem that follows linear
sequence dynamics with the $p$-norm. We provide an upper bound on the policy
regret for this more general formulation and the proof
of~\cref{thm:regret_upper_bound_finite} follows as a special case when $p = 2$.

\begin{restatable}{theorem}{thmregretupperboundfinitep}\label{thm:regret_upper_bound_finite_p}
  Consider an online convex optimization with finite memory problem with
  constant memory length $m$, $(\calX, \calH = \calX^m, A_{\text{finite},m},
  B_{\text{finite},m})$. Assume that the problem follows linear sequence
  dynamics with the $p$-norm for $p \geq 1$. Let the regularizer $R : \calX \to
  \R$ be $\alpha$-strongly-convex and satisfy $|R(x) - R(\tilde{x})| \leq D$ for
  all $x, \tilde{x} \in \calX$.~\cref{alg:ftrl} with step-size $\eta$ satisfies
  \begin{equation*}
    R_T(\texttt{FTRL}) \leq O \left( \sqrt{\frac{D}{\alpha} T L \circfn{L} m^{\frac{p+1}{p}}} \right) \leq O \left( \sqrt{\frac{D}{\alpha} T L^2 m^{\frac{p+2}{p}}} \right).
  \end{equation*}
\end{restatable}

\begin{proof}
  Using~\cref{thm:regret_upper_bound} it suffices to bound $\circfn{L}$ and
  $H_p$ for this problem. Note that $\| A_{\text{finite}}^k \| = 1$ if $k \leq
  m$ and $0$ otherwise. Using this we have
  \begin{equation*}
    H_p = \left( \sum_{k=0}^\infty \left( k \| A_{\text{finite}}^k \| \right)^p \right)^{\frac{1}{p}} = \left( \sum_{k=0}^m k^p \right)^{\frac{1}{p}} \leq O \left( m^{\frac{p+1}{p}} \right).
  \end{equation*}
  This proves the first inequality in the statement of the theorem. The second
  inequality follows from the above and~\cref{thm:lcirc}, which states that
  \begin{equation*}
    \circfn{L} \leq L \left( \sum_{k=0}^\infty \| A_{\text{finite}}^k \|^p \right)^{\frac{1}{p}} = L m^{\frac{1}{p}}. \qedhere
  \end{equation*}
\end{proof}


Finally, we provide an upper bound on the policy regret for the OCO with
$\rho$-discounted infinite memory problem. For simplicity, we consider the case
when the problem follows linear sequence dynamics with the $2$-norm instead of a
general $p$-norm.

\begin{restatable}{theorem}{thmregretupperboundinfinite}\label{thm:regret_upper_bound_infinite}
  Consider an online convex optimization with $\rho$-discounted infinite memory
  problem $(\calX, \calH, A_{\text{infinite}, \rho}, B_{\text{infinite}})$.
  Suppose that the problem follows linear sequence dynamics with the $2$-norm.
  Let the regularizer $R : \calX \to \R$ be $\alpha$-strongly-convex and satisfy
  $|R(x) - R(\tilde{x})| \leq D$ for all $x, \tilde{x} \in
  \calX$.~\cref{alg:ftrl} with step-size $\eta$ satisfies
  \begin{equation*}
    R_T(\texttt{FTRL}) \leq O \left( \sqrt{\frac{D}{\alpha} T L \circfn{L} (1-\rho^2)^{-\frac{3}{2}}} \right) \leq O \left( \sqrt{\frac{D}{\alpha} T L^2 (1-\rho^2)^{-2} } \right) \leq O \left( \sqrt{\frac{D}{\alpha} T L^2 (1-\rho)^{-2} } \right).
  \end{equation*}
\end{restatable}

\begin{proof}
  Using~\cref{thm:regret_upper_bound}, it suffices to bound $\circfn{L}$ and
  $H_p$ for this problem. Recall that $\| A_{\text{infinite},\rho}^k \| =
  \rho^k$. Using this we have
  \begin{equation*}
    H_2 = \left( \sum_{k=0}^\infty \left( k \| A_{\text{finite}}^k \| \right)^2 \right)^{\frac12} = \left( \sum_{k=0}^\infty \left( k \rho^k \right)^2 \right)^{\frac12} \leq (1-\rho^2)^{-\frac{3}{2}}.
  \end{equation*}
  This proves the first inequality in the statement of the theorem. The second
  inequality follows from the above and~\cref{thm:lcirc}, which states that
  \begin{equation*}
    \circfn{L} \leq L \left( \sum_{k=0}^\infty \| A_{\text{infinite},\rho}^k \|^2 \right)^{\frac12} = L (1-\rho^2)^{-\frac{1}{2}}.
  \end{equation*}
  The last inequality follows because $1 - \rho^2 = (1 + \rho) (1 - \rho)$,
  which implies that $1 - \rho \leq 1 - \rho^2 \leq 2 (1 - \rho)$ because $\rho
  \in (0, 1)$.
\end{proof}


\subsection{Existing Regret Bound for OCO with Finite Memory}
\label{subsec:appendix_oco_finite_existing}

In this subsection we provide a detailed comparison of our upper bound on the
policy regret for OCO with finite memory with that of~\citet{anavaHM2015online}.
The material in this subsection comes from Appendix A.2 of their arXiv version
or Appendix C.2 of their conference version.

The existing upper bound on regret is
\begin{equation*}
  O \left( \sqrt{ D T \lambda m^\frac32 } \right),
\end{equation*}
where $D = \max_{x, \tilde{x} \in \calX} |R(x) - R(\tilde{x})|$. Although the
parameter $\lambda$ is defined in terms of dual norms of the gradient of
$\circfn{f}_t$, it is essentially the Lipschitz-continuity constant for
$\circfn{f}_t$: for all $x, \tilde{x} \in \calX$,
\begin{equation*}
  \left\vert \circfn{f}_t(x) - \circfn{f}_t(\tilde{x}) \right\vert \leq \sqrt{\lambda \alpha} \| x - \tilde{x} \|,
\end{equation*}
where $\alpha$ is the strong-convexity parameter of the regularizer $R$ (or
$\sigma$ in the notation of~\citet{anavaHM2015online}). Therefore, the existing
regret bound can be rewritten as
\begin{equation*}
  O \left( \circfn{L} \sqrt{ \frac{D}{\alpha} T m^\frac32 } \right).
\end{equation*}

Our upper bound on the policy regret for OCO with finite
memory~\cref{thm:regret_upper_bound_finite} is
\begin{equation*}
  O \left( \sqrt{ \frac{D}{\alpha} L \circfn{L} T m^\frac32 } \right).
\end{equation*}
Since $\circfn{L} \leq \sqrt{m} L$ by~\cref{thm:lcirc}, this leads to an
improvement by a factor of $m^\frac14$.

\section{Regret Analysis: Lower Bounds}
\label{sec:appendix_regret_lower_bounds}


We first restate~\cref{thm:regret_lower_bound,thm:regret_lower_bound_finite}.

\thmregretlowerbound*

\thmregretlowerboundfinite*

\cref{thm:regret_lower_bound} follows from~\cref{thm:regret_lower_bound_finite}.
However, the lower bound is true for a much broader class of problems as we show
in this section. We first provide a lower bound for a more general formulation
of the OCO with finite memory problem~(\cref{thm:regret_lower_bound_finite_p}).
The proof of~\cref{thm:regret_lower_bound_finite} follows as a special case when
$p = 2$. Then, we provide a lower bound for the OCO with $\rho$-discounted
infinite memory problem~(\cref{thm:regret_lower_bound_infinite}).


The OCO with finite memory problem, as defined in the literature, follows linear
sequence dynamics with the $2$-norm. In this section we consider a more general
version of the OCO with finite memory problem that follows linear sequence
dynamics with the $p$-norm. We provide a lower bound on the policy regret for
this more general formulation and the proof
of~\cref{thm:regret_lower_bound_finite} follows as a special case when $p = 2$.

\begin{restatable}{theorem}{thmregretlowerboundfinitep}\label{thm:regret_lower_bound_finite_p}
  For all $p \geq 1$, there exists an instance of the online convex optimization
  with finite memory problem with constant memory length $m$, $(\calX, \calH =
  \calX^m, A_{\text{finite},m}, B_{\text{finite},m})$, that follows linear
  sequence dynamics with the $p$-norm, and there exist $L$-Lipschitz continuous
  loss functions $\{ f_t : \calH \to \R \}_{t=1}^{T}$ such that the regret of
  any algorithm $\calA$ satisfies
  \begin{equation*}
    R_T(\calA) \geq \Omega \left( \sqrt{T L^2 m^{\frac{p+2}{p}}} \right).
  \end{equation*}
\end{restatable}

\begin{proof}
  Let $\calX = [-1, 1]$ and consider an OCO with finite memory problem with
  constant memory length $m$, $(\calX, \calH = \calX^m, A_{\text{finite},m},
  B_{\text{finite},m})$, that follows linear sequence dynamics with the
  $p$-norm. For simplicity, assume that $T$ is a multiple of $m$ (otherwise, the
  same proof works but with slightly more tedious bookkeeping) and that $L = 1$
  (otherwise, multiply the functions $f_t$ defined below by $L$).

  Divide the $T$ rounds into $N = \frac{T}{m}$ blocks of $m$ rounds each. Sample
  $N$ independent Rademacher random variables $\{ \eps_1, \dots, \eps_N \}$,
  where each $\eps_i$ is equal to $\pm 1$ with probability $\frac12$. Recall
  that $h_t = (x_t, \dots, x_{t-m+1})$. Define the loss functions $\{ f_t
  \}_{t=1}^{T}$ as follows. (See~\cref{fig:lower_bound_oco_finite_appendix} for
  an illustration.) If $t \leq m$, let $f_t = 0$. Otherwise, let
  \begin{figure}[tb]
    \centering
    \begin{tikzpicture}
      
      \draw[thick, -Triangle] (0,0) -- (13cm,0) node[font=\scriptsize,below left=3pt and -8pt]{Time};
      
      \foreach \x in {0,1,...,11}
      \draw (\x cm, 3pt) -- (\x cm, -3pt);

      \foreach \x in {0,3,...,11}
      \draw[very thick] (\x cm, 3pt) -- (\x cm, -3pt);

      \foreach \x/\descr in {0/x_1,1/x_2,2/x_3,3/x_4,4/x_5,5/x_6,6/x_7,7/x_8,8/x_9,9/x_{10},10/x_{11},11/x_{12}}
      \node[font=\scriptsize, text height=1.75ex, text depth=.5ex] at (\x+.5, -.3) {$\descr$};

      \foreach \x/\descr in {0/\eps_1,3/\eps_2,6/\eps_3,9/\eps_4}
      \node[font=\scriptsize, text height=1.75ex, text depth=.5ex] at (\x+.5, .3) {$\descr$};
      
      \foreach \x in {1,2,3}
      \filldraw[draw=black,fill=SeaGreen] (\x,-.7) rectangle (\x+1, -.9);
      \node[font=\scriptsize, text height=1.75ex, text depth=.5ex] at (6, -.8) {$f_4(h_4) = \eps_2 \  3^{-\frac12} \  (x_2 + x_3 + x_4)$};

      \foreach \x in {2,3}
      \filldraw[draw=black,fill=SeaGreen] (\x,-1.1) rectangle (\x+1, -1.3);
      \foreach \x in {4}
      \filldraw[draw=black,fill=Orange] (\x,-1.1) rectangle (\x+1, -1.3);
      \node[font=\scriptsize, text height=1.75ex, text depth=.5ex] at (6.7, -1.2) {$f_5(h_5) = \eps_2 \  3^{-\frac12} \  (x_3 + x_4)$};

      \foreach \x in {3}
      \filldraw[draw=black,fill=SeaGreen] (\x,-1.5) rectangle (\x+1, -1.7);
      \foreach \x in {4,5}
      \filldraw[draw=black,fill=Orange] (\x,-1.5) rectangle (\x+1, -1.7);
      \node[font=\scriptsize, text height=1.75ex, text depth=.5ex] at (7.4, -1.6) {$f_6(h_6) = \eps_2 \  3^{-\frac12} \  (x_4)$};

      \node[font=\scriptsize, text height=1.75ex, text depth=.5ex] at (1, 1) {$T = 12, m = 3, L = 1$.};
      \node[font=\scriptsize, text height=1.75ex, text depth=.5ex] at (5.5, 1) {Resample every $m$ rounds.};
      \draw (3.5, 1) -- (4, 1);
      \draw (3.5, .6) -- (3.5, 1);
      \draw[-Triangle] (3.5, .6) -- (3.5, 0.5);
      
    \end{tikzpicture}
    \caption{An illustration of the loss functions $f_t$ for the OCO with finite memory lower bound. Suppose $T = 12, m = 3, L = 1$, and $p = 2$. Time is divided into blocks of size $m = 3$. Consider round $t = 5$. The history is $h_5 = (x_3, x_4, x_5)$. The loss function $f_5(h_5)$ is a product of three terms: a random sign $\eps_2$ sampled for the block that round $5$ belongs to, namely, block $2$; a scaling factor of $m^{-\frac12}$; a sum over the decisions in the history excluding those that were chosen after observing $\eps_2$, i.e., a sum over $x_3$ and $x_4$, excluding $x_5$.}
    \label{fig:lower_bound_oco_finite_appendix}
  \end{figure}
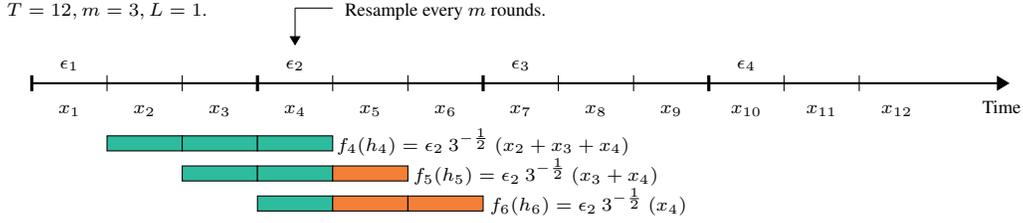
  \begin{align*}
    f_t(h_t)
    &= \eps_{\ceil{\frac{t}{m}}} m^{\frac{1-p}{p}} \sum_{k=0}^{m - 1 - (t - m \floor{\frac{t}{m}} - 1)} x_{m \floor{\frac{t}{m}} + 1 - k} \\
    &= \eps_{\ceil{\frac{t}{m}}} m^{\frac{1-p}{p}} \left( x_{t-m+1} + \dots + x_{m \floor{\frac{t}{m}} + 1} \right).
  \end{align*}
  In words, the loss in the first $m$ rounds is equal to $0$. Thereafter, in
  round $t$ the loss is equal to a random sign $\eps_{\ceil{\frac{t}{m}}}$,
  which is \emph{fixed for that block}, times a scaling factor, which is chosen
  according to the $p$-norm to ensure that the Lipschitz constant $L$ is at most
  $1$, times a sum of a \emph{subset} of past decisions in the history $h_t =
  (x_t, \dots, x_{t-m+1})$. This subset consists of all past decisions until and
  including the first decision of the current block, which is the decision in
  round $m \floor{\frac{t}{m}} + 1$.

  The functions $f_t$ are linear, so they are convex. In order to show that they
  satisfy Assumptions~\ref{ass:convexity} and~\ref{ass:lipschitz}, it remains to
  show that they are $1$-Lipschitz continuous. Let $h = (x^{(1)}, \dots,
  x^{(m)})$ and $\tilde{h} = (\tilde{x}^{(1)}, \dots, \tilde{x}^{(m)})$ be
  arbitrary elements of $\calH = \calX^m$. We have
  \begin{align*}
    & \left\vert f_t(h) - f_t(\tilde{h}) \right\vert \\
    &\leq \left\vert \eps_{\ceil{\frac{t}{m}}} m^{\frac{1-p}{p}} \left( (x^{(1)} - \tilde{x}^{(1)}) + \dots + (x^{(m)} - \tilde{x}^{(m)}) \right) \right\vert \\
    &\leq m^{\frac{1-p}{p}} \left\vert (x^{(1)} - \tilde{x}^{(1)}) + \dots + (x^{(m)} - \tilde{x}^{(m)}) \right\vert & \text{ because } \eps_{\ceil{\frac{t}{m}}} \in \{ -1, +1 \} \\
    &\leq m^{\frac{1-p}{p}} m^{1 - \frac{1}{p}} \left( \sum_{k=1}^m \left\vert x^{(k)} - \tilde{x}^{(k)} \right\vert^p \right)^{\frac{1}{p}} & \text{ by H\"older's inequality} \\
    &= \| h - \tilde{h} \|_\calH,
  \end{align*}
  where the last equality follows because of our assumption that the problem
  that follows linear sequence dynamics with the $p$-norm.

  First we will show that the total expected loss of any algorithm is $0$, where
  the expectation is with respect to the randomness in the choice of $\{ \eps_1,
  \dots, \eps_N \}$. The total loss in the first block is $0$ because $f_t = 0$
  for $t \in [m]$. For each subsequent block $n \in \{ 2, \dots, N \}$, the
  total loss in block $n$ depends on the algorithm's choices made \emph{before}
  observing $\eps_{n}$, namely, $\{x_{(n-2) m + 2}, \dots, x_{(n-1) m + 1} \}$.
  Since $\eps_{n}$ is equal to $\pm 1$ with probability $\frac12$, the expected
  loss of any algorithm in a block is equal to $0$ and the total expected loss
  is also equal to $0$.

  Now we will show that the expected loss of the benchmark is at most
  \begin{equation*}
    -O \left( \sqrt{T m^{\frac{p+2}{p}}} \right),
  \end{equation*}
  where the expectation is with respect to the randomness in the choice of $\{
  \eps_1, \dots, \eps_N \}$. We have
  \begin{align*}
    \E \left[ \min_{x \in \calX} \sum_{t=1}^T \circfn{f}_t(x) \right]
    &= \E \left[ \min_{x \in \calX} \sum_{n=2}^N \sum_{t=(n-1) m + 1}^{nm} \circfn{f}_t(x) \right] \\
    &= \E \left[ \min_{x \in \calX} \sum_{n=2}^N \sum_{t=(n-1) m + 1}^{nm} \eps_n m^{\frac{1-p}{p}} \times x \times (m - (t - (n-1) m - 1)) \right].
    \intertext{The first equality follows from first summing over blocks and then summing over the rounds in that block. The second equality follows from the definitions of $f_t$ above and of $\circfn{f}_t$~(\cref{def:circfn}). By the defintion of $\circfn{f}_t$, the history $h_t$ consists of $m$ copies of $x$ for $t \geq m$.. By the definition of $f_t$, which sums over all past decisions until the first round of the current block, we have that within a block the sum first extends over $m$ copies of $x$ (in the first round of the block), then $m-1$ copies of $x$ (in the second round of the block), and so on until the last round of the block. So, we have}
    \E \left[ \min_{x \in \calX} \sum_{t=1}^T \circfn{f}_t(x) \right]
    &= \E \left[ \min_{x \in \calX} \sum_{n=2}^N \sum_{t=(n-1) m + 1}^{nm} \eps_n m^{\frac{1-p}{p}} \times x \times (m - (t - (n-1) m - 1)) \right] \\
    &= \E \left[ \min_{x \in \calX} \sum_{n=2}^N \sum_{k=0}^{m-1} \eps_n m^{\frac{1-p}{p}} \times x \times (m - k) \right] \\
    &= m^{\frac{1-p}{p}} \frac{m^2 + m}{2} \  \E \left[ \min_{x \in \calX} \sum_{n=2}^N \eps_n x \right] \\
    &= m^{\frac{1-p}{p}} \frac{m^2 + m}{2} \  \E \left[ \min_{x \in \{-1, 1\}} \sum_{n=2}^N \eps_n x \right] \\
    &= m^{\frac{1-p}{p}} \frac{m^2 + m}{2} \  \E \left[ \frac12 \sum_{n=2}^N \eps_n (-1 + 1) - \frac12 \left\vert \sum_{n=2}^N \eps_n (-1 - 1) \right\vert \right], \\
    \intertext{where the second-last equality follows because the minima of a linear function over an interval is at one of the endpoints and the last equality follows because $\min\{x, y\} = \frac12 (x + y) - \frac12 \left\vert x - y \right\vert$. Since $\eps_n$ are Rademacher random variables equal to $\pm 1$ with probability $\frac12$, we can simplify the above as}
    \E \left[ \min_{x \in \calX} \sum_{t=1}^T \circfn{f}_t(x) \right]
    &= m^{\frac{1-p}{p}} \frac{m^2 + m}{2} \  \E \left[ - \frac12 \left\vert \sum_{n=2}^N -2 \eps_n \right\vert \right] \\
    &= m^{\frac{1-p}{p}} \frac{m^2 + m}{2} \  \E \left[ - \left\vert \sum_{n=2}^N \eps_n \right\vert \right] \\
    &= -  m^{\frac{1-p}{p}} \frac{m^2 + m}{2} \E \left[ \left\vert \sum_{n=2}^N \eps_n \right\vert \right] \\
    &\leq -  m^{\frac{1-p}{p}} \frac{m^2 + m}{2} \sqrt{N},
    \intertext{where the last inequality follows from Khintchine's inequality. Using the definition $N = \frac{T}{m}$, we have}
    \E \left[ \min_{x \in \calX} \sum_{t=1}^T \circfn{f}_t(x) \right]
    &\leq -  m^{\frac{1-p}{p}} \frac{m^2 + m}{2} \sqrt{\frac{T}{m}} \\
    &= - \frac12 \sqrt{T} \left(m^{\frac{3}{2} + \frac{1-p}{p}} + m^{\frac12 + \frac{1-p}{p}} \right) \\
    &\leq - O\left( \sqrt{T} m^{\frac{3}{2} + \frac{1-p}{p}} \right) \\
    &= - O\left( \sqrt{T} m^{\frac{p+2}{2p}} \right) \\
    &= - O\left( \sqrt{T m^{\frac{p+2}{p}}} \right).
  \end{align*}

  Therefore, we have
  \begin{equation*}
    \E_{\eps_1, \dots, \eps_N} \left[ R_T(\texttt{FTRL}) \right]
    = \E \left[ \sum_{t=1}^T f_t(h_t) \right] - \E \left[ \min_{x \in \calX} \sum_{t=1}^T \circfn{f}_t(x) \right]
    \geq \Omega \left( \sqrt{T m^{\frac{p+2}{p}}} \right).
  \end{equation*}
  This completes the proof.
\end{proof}


Now we provide a lower bound for the OCO with $\rho$-discounted infinite memory
problem. For simplicity, we consider the case when the problem follows linear
sequence dynamics with the $2$-norm instead of a general $p$-norm.

\begin{restatable}{theorem}{thmregretlowerboundinfinite}\label{thm:regret_lower_bound_infinite}
  Let $\rho \in [\frac12, 1)$. There exists an instance of the online convex
  optimization with $\rho$-discounted infinite memory problem, $(\calX, \calH,
  A_{\text{infinite}, \rho}, B_{\text{infinite}})$, that follows linear sequence
  dynamics with the $2$-norm and there exist $L$-Lipschitz continuous loss
  functions $\{ f_t : \calH \to \R \}_{t=1}^T$ such that the regret of any
  algorithm $\calA$ satisfies
  \begin{equation*}
    R_T(\calA) \geq \Omega \left( \sqrt{T L^2 (1-\rho)^{-2}} \right).
  \end{equation*}
\end{restatable}

The proof is very similar to that of~\cref{thm:regret_lower_bound_finite_p} with
slight adjustments to account for a $\rho$-discounted infinite memory instead of
a finite memory of constant size $m$.

\begin{proof}
  Let $\calX = [-1, 1]$ and consider an OCO with infinite memory problem with
  discount factor $\rho$, $(\calX, \calH, A_{\text{infinite}, \rho},
  B_{\text{infinite}})$, that follows linear sequence dynamics with the
  $2$-norm. For simplicity, assume that $T$ is a multiple of $(1 - \rho)^{-1}$
  (otherwise, the same proof works but with slightly more tedious bookkeeping)
  and that $L = 1$ (otherwise, multiply the functions $f_t$ defined below by
  $L$).

  Define $m = (1 - \rho)^{-1}$. Divide the $T$ rounds into $N = \frac{T}{m}$
  blocks of $m$ rounds each. Sample $N$ independent Rademacher random variables
  $\{ \eps_1, \dots, \eps_N \}$, where each $\eps_i$ is equal to $\pm 1$ with
  probability $\frac12$. Recall that $h_t = (x_t, \rho x_{t-1}, \dots,
  \rho^{t-1} x_1, 0, \dots)$. Define the loss functions $\{ f_t \}_{t=1}^{T}$ as
  follows. If $t \leq m$, let $f_t = 0$. Otherwise, let
  \begin{equation*}
    f_t(h_t)
    = \eps_{\ceil{\frac{t}{m}}} m^{-\frac12} \sum_{k=0}^{m - 1} \rho^{k + t - m \floor{\frac{t}{m}}-1} x_{m \floor{\frac{t}{m}} + 1 - k}.
  \end{equation*}

  The functions $f_t$ are linear, so they are convex. In order to show that they
  satisfy Assumptions~\ref{ass:convexity} and~\ref{ass:lipschitz}, it remains to
  show that they are $1$-Lipschitz continuous. Let $h = (x^{(1)}, \rho x^{(2)},
  \dots)$ and $\tilde{h} = (\tilde{x}^{(1)}, \rho \tilde{x}^{(2)}, \dots)$ be
  arbitrary elements of $\calH$. We have
  \begin{align*}
    & \left\vert f_t(h) - f_t(\tilde{h}) \right\vert \\
    &\leq \left\vert \eps_{\ceil{\frac{t}{m}}} m^{-\frac12} \sum_{k=1}^m \rho^{k-1} \left( x^{(k)} - \tilde{x}^{(k)} \right) \right\vert \\
    &\leq m^{-\frac12} \left\vert \sum_{k=1}^m \rho^{k-1} \left( x^{(k)} - \tilde{x}^{(k)} \right) \right\vert & \text{ because } \eps_{\ceil{\frac{t}{m}}} \in \{ -1, +1 \} \\
    &\leq m^{-\frac12} m^{\frac12} \left( \sum_{k=1}^m \rho^{2(k-1)} \left\vert x^{(k)} - \tilde{x}^{(k)} \right\vert^2 \right)^\frac12 & \text{ by H\"older's inequality} \\
    &\leq \| h - \tilde{h} \|_\calH,
  \end{align*}
  where the last equality follows because the follows linear sequence dynamics
  with the $2$-norm.

  First we will show that the total expected loss of any algorithm is $0$, where
  the expectation is with respect to the randomness in the choice of $\{ \eps_1,
  \dots, \eps_N \}$. The total loss in the first block is $0$ because $f_t = 0$
  for $t \in [m]$. For each subsequent block $n \in \{ 2, \dots, N \}$, the
  total loss in block $n$ depends on the algorithm's choices made \emph{before}
  observing $\eps_{n}$, namely, $\{x_{(n-2) m + 2}, \dots, x_{(n-1) m + 1} \}$.
  Since $\eps_{n}$ is equal to $\pm 1$ with probability $\frac12$, the expected
  loss of any algorithm in a block is equal to $0$ and the total expected loss
  is also equal to $0$.

  Now we will show that the expected loss of the benchmark is at most
  \begin{equation*}
    -O \left( \sqrt{T (1-\rho)^{-2}} \right),
  \end{equation*}
  where the expectation is with respect to the randomness in the choice of $\{
  \eps_1, \dots, \eps_N \}$. We have
  \begin{align*}
    \E \left[ \min_{x \in \calX} \sum_{t=1}^T \circfn{f}_t(x) \right]
    &= \E \left[ \min_{x \in \calX} \sum_{n=2}^N \sum_{t=(n-1) m + 1}^{nm} \circfn{f}_t(x) \right] \\
    &= \E \left[ \min_{x \in \calX} \sum_{n=2}^N \sum_{t=(n-1) m + 1}^{nm}  \eps_n m^{-\frac12} \sum_{k=0}^{m - 1} \rho^{k + t - (n-1)m -1} x \right] \\
    &= m^{-\frac12} \E \left[ \min_{x \in \calX} \sum_{n=2}^N \eps_n x \sum_{t=(n-1) m + 1}^{nm} \rho^{t - (n-1)m -1} \sum_{k=0}^{m - 1} \rho^k \right] \\
    &= m^{-\frac12} \frac{1 - \rho^m}{1 - \rho} \E \left[ \min_{x \in \calX} \sum_{n=2}^N \eps_n x \sum_{t=(n-1) m + 1}^{nm} \rho^{t - (n-1)m -1} \right] \\
    &= m^{-\frac12} \left(\frac{1 - \rho^m}{1 - \rho}\right)^2 \underbrace{\E \left[ \min_{x \in \calX} \sum_{n=2}^N \eps_n x \right]}_{(a)}.
    \intertext{The term (a) above can be bounded above by $-\sqrt{N}$ as in the proof of~\cref{thm:regret_lower_bound_finite_p} using Khintchine's inequality. Therefore, using that $N = \frac{T}{m}$ and $m = (1-\rho)^{-1}$ we have}
    \E \left[ \min_{x \in \calX} \sum_{t=1}^T \circfn{f}_t(x) \right]
    &\leq -m^{-\frac12} \left(\frac{1 - \rho^m}{1 - \rho}\right)^2 \sqrt{N} \\
    &\leq -(1-\rho)^{\frac12} \left(\frac{1 - \rho^m}{1 - \rho}\right)^2 \sqrt{T (1-\rho)} \\
    &= -\sqrt{T} \frac{(1 - \rho^m)^2}{1-\rho} \\
    &= -\sqrt{T (1-\rho)^{-2}} (1 - \rho^m)^2 \\
    &\leq - O \left( \sqrt{T (1-\rho)^{-2}} \right),
  \end{align*}
  where the last inequality follows from the assumption that $\rho \in [\frac12,
  1)$ and the following argument:
  \begin{align*}
    \rho^m &= (1 - (1 - \rho))^m = (1 - (1 - \rho))^{\frac{1}{1 - \rho}} \leq \frac{1}{e} \\
    \Rightarrow (1 - \rho^m) &\geq 1 - \frac{1}{e} \\
    \Rightarrow (1 - \rho^m)^2 &\geq \left( 1 - \frac{1}{e} \right)^2 \\
    \Rightarrow -(1 - \rho^m)^2 &\leq -\left( 1 - \frac{1}{e} \right)^2 \\
  \end{align*}
  This completes the proof.
\end{proof}

\section{Online Linear Control}
\label{sec:appendix_olc}


\subsection{Formulation as OCO with Unbounded Memory}\label{subsec:appendix_olc_formulation}

Now we formulate the online linear control problem in our framework by defining
the decision space $\calX$, the history space $\calH$, and the linear operators
$A : \calH \to \calH$ and $B : \calW \to \calH$. Then, we define the functions
$f_t : \calH \to \R$ in terms of $c_t$ and finally, prove an upper bound on the
policy regret. For notational convenience, let $(M^{[s]})$ and $(Y_k)$ denote
the sequences $(M^{[1]}, M^{[2]}, \dots)$ and $(Y_0, Y_1, \dots)$ respectively.

Recall that we fix $K \in \calK$ to be an arbitrary $(\kappa, \rho)$-strongly
stable linear controller and consider the disturbance-action controller policy
class $\calM_K$~(\cref{def:disturbance_action_controller}). For the rest of this
paper let $\wt{F} = F - GK$.
The first step is a change of variables with respect to the control inputs from
linear controllers to DACs and the second is a corresponding change of variables
for the state.
Define the decision space $\calX$ as
\begin{align}
  \calX &= \{ M = (M^{[s]}) : M^{[s]} \in \R^{d \times d}, \| M^{[s]} \|_2 \leq \kappa^4 \rho^s \} \label{eq:appendix_olc_calx}
  \intertext{with}
  \| M \|_\calX &= \sqrt{\sum_{s=1}^\infty \rho^{-s} \| M^{[s]} \|_F^2}. \label{eq:appendix_olc_calx_norm}
\end{align}
Define the history space $\calH$ to be the set consisting of sequences $h =
(Y_k)$, where $Y_0 \in \calX$ and $Y_k = \wt{F}^{k-1} G X_k$ for $X_k \in \calX,
k \geq 1$ with
\begin{align}
  \| h \|_\calH &= \sqrt{ \sum_{k=0}^\infty \xi_k^2 \| Y_k \|_\calX^2}, \label{eq:appendix_olc_calh_norm}
  \intertext{where the weights $(\xi_k)$ are nonnegative real numbers defined as}
  \xi &= (1, 1, 1, \rho^{-\frac12}, \rho^{-1}, \rho^{-\frac32}, \dots). \label{eq:appendix_olc_xi}
\end{align}

Define the linear operators $A : \calH \to \calH$ and $B : \calW \to \calH$ as
\begin{equation*}
  A((Y_0, Y_1, \dots)) = (0, G Y_0, \wt{F} Y_1, \wt{F} Y_2, \dots)
  \quad
  \text{ and }
  \quad
  B(M) = (M, 0, 0, \dots).
\end{equation*}
Note that the problem follows linear sequence dynamics with the $\xi$-weighted
$2$-norm~(\cref{def:sequence}), where $\xi$ is defined above
in~\cref{eq:appendix_olc_xi}. The weights in the weighted norms on $\calX$ and
$\calH$ increase exponentially. However, the norms $\| M^{[s]} \|_F^2$ and $\|
\wt{F}^{k-1} G \|_F^2$ decrease exponentially as well: by definition of
$M^{[s]}$ in~\cref{eq:appendix_olc_calx_norm} and the assumption on $\wt{F} = F
- G K$ for $K \in \calK$. Leveraging this exponential decrease in $\| M^{[s]}
\|_F^2$ and $\| \wt{F}^{k-1} G \|_F^2$ to define exponentially increasing
weights turns out to be crucial for deriving our regret bounds that are stronger
than existing results. Furthermore, the choice to have $\xi_p = 1$ for $p \in \{
1, 2 \}$ in addition to $p = 0$ (as required by~\cref{def:sequence}) might seem
like a small detail, but this also turns out to be crucial for avoiding
unnecessary factors of $\rho^{-1}$ in the regret bounds.

Recall that the loss functions in the online linear control problem are
$c_t(s_t, u_t)$, where $s_t$ and $u_t$ are the state and control at round $t$.
Now we will show how to construct the functions $f_t : \calH \to \R$ that
correspond to $c_t(s_t, u_t)$. By definition, given a sequence of decisions
$(M_0, \dots, M_t)$, the history at the end of round $t$ is given by
\begin{equation*}
  h_t = (M_t, G M_{t-1}, \wt{F} G M_{t-2}, \dots, \wt{F}^{t-1} G M_0, 0, \dots).
\end{equation*}

A simple inductive argument shows that the state and control in
round $t$ can be written as
\begin{align}
  s_t &= \wt{F}^t s_0 + \sum_{k=0}^{t-1} \sum_{s=1}^{k+1} \wt{F}^{t-k-1} G M_k^{[s]} w_{k-s} + w_{t-1}, \label{eq:appendix_olc_st} \\
  u_t &= -K s_t + \sum_{s=1}^{t+1} M_t^{[s]} w_{t-s}. \label{eq:appendix_olc_ut}
\end{align}
Define the functions $f_t : \calH \to \R$ by $f_t(h) = c_t(s, u)$, where $s$ and
$u$ are the state and control determined by the history as above. Note that
$f_t$ is parameterized by the past disturbances. Since the state and control are
linear functions of the history and $c_t$ is convex, this implies that $f_t$ is
convex.

With the above formulation and the fact that the class of disturbance-action
controllers is a superset of the class of $(\kappa,\rho)$-strongly-stable linear
controllers, we have that the policy regret for the online linear control
problem is at most
\begin{equation*}
  \sum_{t=0}^{T-1} f_t(h_t) - \min_{M \in \calX} \sum_{t=0}^{T-1} \circfn{f}_t(M).
\end{equation*}

This completes the specification of the online convex optimization with
unbounded memory problem, $(\calX, \calH, A, B)$, corresponding to the online
linear control problem. Using~\cref{alg:ftrl} and~\cref{thm:regret_upper_bound}
we can upper bound the above by
\begin{equation*}
  O \left( \sqrt{\frac{D}{\alpha} T L \circfn{L} H_{2}} \right),
\end{equation*}
where $L$ is the Lipschitz constant of $f_t$, $\circfn{L}$ is the Lipschitz
constant of $\circfn{f}_t$, $H_{2}$ is the $2$-effective memory capacity, and $D
= \max_{x, \tilde{x} \in \calX} |R(x) - R(\tilde{x})|$ for an
$\alpha$-strongly-convex regularizer $R : \calX \to \R$. In the next subsection
we bound these quantities in terms of the problem parameters of the online
linear control problem. We use $O(\cdot)$ to hide absolute constants.


\subsection{Regret Analysis}\label{subsec:appendix_olc_regret_analysis}

We use the following standard facts about matrix norms.
\begin{lemma}\label{lemma:matrix_norms}
  Let $M, N \in \R^{d \times d}$. Then,
  \begin{enumerate}
    \item $ \| M \|_2 \leq \| M \|_F \leq \sqrt{d} \| M \|_2$.
    \item $ \| M N \|_F \leq \| M \|_2 \| N \|_F$.
  \end{enumerate}
\end{lemma}

\begin{proof}
  Part 1 can be found in, for example,~\citet[Section 2.3.2]{golubV1996matrix}.
  Letting $N_j$ denote the $j$-th column of $N$, part 2 follows from
  \begin{align*}
    \| M N \|_F^2 = \sum_{j=1}^d \| M N_j \|_2^2 \leq \| M \|_2^2 \sum_{j=1}^d \| N_j \|_2^2 = \| M \|_2^2 \| N \|_F^2.
  \end{align*}
  This completes the proof.
\end{proof}

\begin{lemma}\label{lemma:appendix_olc_Ak}
  For $s \geq 2$, the operator norm $\| A^s \|$ is bounded above as
  \begin{equation*}
    \left\| A^s \right\| \leq O \left( \kappa^4 \rho^{\frac{s}{2}} \right).
  \end{equation*}
\end{lemma}

\begin{proof}
  Recall the definition of $\calH$ and $\| \cdot
  \|_\calH$~(\cref{eq:appendix_olc_calh_norm}). Let
  \begin{equation*}
    (Y_0, Y_1, \dots) = (Y_0, G X_1, \wt{F} G X_2, \wt{F}^2 G X_3, \dots)
  \end{equation*}
  be an element of $\calH$ with unit norm, i.e.,
  \begin{equation*}
    \sqrt{ \sum_{k=0}^\infty \xi_k^2 \| Y_k \|_\calX^2 } = 1,
  \end{equation*}
  where the weights $(\xi_k)$ are defined in~\cref{eq:appendix_olc_xi}. Note
  that $\xi_p = 1$ for $p = 0,1$ and $\xi_p^2 = \rho^{-p+2}$ for $p =
  2,3,\dots$. From the definition of the operator $A$ and for $s \geq 2$, we
  have
  \begin{equation*}
    A^s((Y_0, Y_1, \dots))
    = (0, \dots, 0, \wt{F}^{s-1} G Y_0, \wt{F}^{s} G X_1, \wt{F}^{s+1} G X_2, \dots).
  \end{equation*}
  Now we bound $\| A^s \|$ as follows. By definition of $A^s$ and $\| \cdot
  \|_\calH$~(\cref{eq:appendix_olc_calh_norm}), and part 2
  of~\cref{lemma:matrix_norms}, we have
  \begin{align*}
    \left\| A^s((Y_0, Y_1, \dots)) \right\|
    &= \sqrt{\rho^{-s+2} \| \wt{F}^{s-1} G Y_0 \|_\calX^2 + \sum_{k=1}^\infty \rho^{-s-k+2} \| \wt{F}^{s+k-1} G X_k \|_\calX^2} \\
    &\leq \sqrt{\rho^{-s+2} \| \wt{F}^{s-1} G \|_2^2 \| Y_0 \|_\calX^2 + \sum_{k=1}^\infty \rho^{-s-k+2} \| \wt{F}^{s-1} \|_2^2 \| \wt{F} \|_2^2 \| \wt{F}^{k-1} G X_k \|_\calX^2} \\
    &\leq \rho^{-\frac{s}{2}} \| \wt{F}^{s-1} \|_2 \sqrt{ \rho^2 \| G \|_2^2 \| Y_0 \|_\calX^2 + \sum_{k=1}^\infty \rho^{-k+2} \| \wt{F} \|_2^2 \| \wt{F}^{k-1} G X_k \|_\calX^2} \\
    &= \rho^{-\frac{s}{2}} \| \wt{F}^{s-1} \|_2 \sqrt{ \rho^2 \| G \|_2^2 \| Y_0 \|_\calX^2 + \sum_{k=1}^\infty \rho^{-k+2} \| \wt{F} \|_2^2 \| Y_k \|_\calX^2}. \\
    \intertext{Using our assumptions that $\| G \|_2 \leq \kappa$ and $\| \wt{F} \|_2 \leq \kappa^2 \rho$, we have}
    \left\| A^s((Y_0, Y_1, \dots)) \right\|
    &\leq \rho^{-\frac{s}{2}} \| \wt{F}^{s-1} \|_2 \sqrt{ \rho^2 \kappa^2 \| Y_0 \|_\calX^2 + \sum_{k=1}^\infty \rho^{-k+2} \kappa^4 \rho^2 \| Y_k \|_\calX^2} \\
    &\leq \rho^{-\frac{s}{2}} \rho \kappa^2 \| \wt{F}^{s-1} \|_2 \sqrt{ \| Y_0 \|_\calX^2 + \sum_{k=1}^\infty \rho^{-k+2} \| Y_k \|_\calX^2} \\
    &\leq \rho^{-\frac{s}{2}} \rho \kappa^2 \kappa^2 \rho^{s-1} \sqrt{ \| Y_0 \|_\calX^2 + \sum_{k=1}^\infty \rho^{-k+2} \| Y_k \|_\calX^2} \\
    &= \kappa^4 \rho^{\frac{s}{2}} \sqrt{ \| Y_0 \|_\calX^2 + \sum_{k=1}^\infty \rho^{-k+2} \| Y_k \|_\calX^2}. \\
    \intertext{Using $\rho^{-1+2} = \rho < 1$ for $k = 1$ in the above sum, the definition of $(\xi_k)$, and our assumption that $(Y_0, Y_1 \dots)$ has unit norm, we have}
    \left\| A^s((Y_0, Y_1, \dots)) \right\|
    &\leq \kappa^4 \rho^{\frac{s}{2}} \sqrt{ \xi_0^2 \| Y_0 \|_\calX^2 + \sum_{k=1}^\infty \xi_k^2 \| Y_k \|_\calX^2}
    = \kappa^4 \rho^{\frac{s}{2}}.
  \end{align*}
  This completes the proof.
\end{proof}

\begin{lemma}\label{lemma:appendix_olc_H2}
  The $2$-effective memory capacity is bounded above as
  \begin{equation*}
    H_{2} \leq O \left( \kappa^4 (1 - \rho)^{-\frac32} \right).
  \end{equation*}
\end{lemma}

\begin{proof}
  Using~\cref{lemma:appendix_olc_Ak} to bound $\| A^k \|$ for $k \geq 2$, we
  have
  \begin{equation*}
    H_{2}
    = \sqrt{ \sum_{k=0}^\infty k^2 \| A^k \|^2 }
    \leq O \left( \sqrt{ \sum_{k=2}^\infty k^2 \kappa^{8} \rho^{k} } \right)
    \leq O \left( \kappa^4 (1 - \rho)^{-\frac32} \right). \qedhere
  \end{equation*}
\end{proof}

\begin{lemma}\label{lemma:appendix_olc_D}
  Suppose $R : \calX \to \R$ is defined by $R(M) = \frac12 \| M \|_\calX^2$.
  Then, it is $1$-strongly-convex and $D = \max_{M, \wt{M} \in \calX} |R(M) -
  R(\wt{M})| \leq d \kappa^8 (1 - \rho)^{-1}$.
\end{lemma}

\begin{proof}
  Note that $R$ is $1$-strongly-convex by definition. Using part 1
  of~\cref{lemma:matrix_norms} and the definition of
  $\calX$~(\cref{eq:appendix_olc_calx}), we have for all $M, \wt{M} \in \calX$,
  \begin{align*}
    D
    &= \max_{M, \wt{M} \in \calX} |R(M) - R(\wt{M})| \\
    &= \max_{M, \wt{M} \in \calX} \left| \frac12 \| M \|_\calX^2 - \frac12 \| \wt{M} \|_\calX^2 \right| \\
    &\leq \max_{M \in \calX} \| M \|_\calX^2 \\
    &= \max_{M \in \calX} \sum_{s=1}^\infty \rho^{-s} \| M^{[s]} \|_F^2 & \text{ by~\cref{eq:appendix_olc_calx_norm}} \\
    &\leq \max_{M \in \calX} \sum_{s=1}^\infty \rho^{-s} d \| M^{[s]} \|_2^2 & \text{ by~\cref{lemma:matrix_norms}} \\
    &\leq \sum_{s=1}^\infty \rho^{-s} d \kappa^8 \rho^{2s} & \text{ by~\cref{eq:appendix_olc_calx}} \\
    &\leq d \kappa^8 (1 - \rho)^{-1}.
  \end{align*}
  This completes the proof.
\end{proof}

\begin{lemma}\label{lemma:appendix_olc_dx}
  We can bound the norm of the state and control at time $t$ as
  \begin{equation*}
    \max \{ \| s_t \|_2, \| u_t \|_2 \}
    \leq D_\calX = O \left( W \kappa^8 (1 - \rho)^{-2} \right).
  \end{equation*}
\end{lemma}

\begin{proof}
  We can bound the norm of $s_t$ and $u_t$
  using~\cref{eq:appendix_olc_st,eq:appendix_olc_ut} as
  \begin{align*}
    \| s_t \|_2
    &\leq \left\| \wt{F}^t s_0 + \sum_{k=0}^{t-1} \sum_{s=1}^{k+1} \wt{F}^{t-k-1} G M_k^{[s]} w_{k-s} + w_{t-1} \right\|_2 \\
    &\leq \kappa^2 \rho^t + W + \sum_{k=0}^{t-1} \sum_{s=1}^{k+1} \kappa^2 \rho^{t-k-1} \kappa \kappa^4 \rho^s W \\
    &\leq \kappa^2 + W + W \kappa^7 (1-\rho)^{-2} \\
    &\leq O \left( W \kappa^7 (1-\rho)^{-2} \right). \\
    \| u_t \|_2
    & \leq \left\| K s_t + \sum_{s=1}^{t+1} M_t^{[s]} w_{t-s} \right\|_2 \\
    &\leq O \left( W \kappa^8 (1 - \rho)^{-2} \right) + \sum_{s=1}^{t+1} W \kappa^4 \rho^s \\
    &\leq O \left( W \kappa^8 (1 - \rho)^{-2} \right).
  \end{align*}
  Above, we used the assumptions that $\kappa, W \geq 1$. This completes the
  proof.
\end{proof}

\begin{lemma}\label{lemma:appendix_olc_L}
  The Lipschitz constant of $f_t$ can be bounded above as
  \begin{equation*}
    L \leq O \left( L_0 D_\calX W \kappa (1 - \rho)^{-1} \right),
  \end{equation*}
  where $D_\calX$ is defined in~\cref{lemma:appendix_olc_dx}.
\end{lemma}

\begin{proof}
  Let $(M_0, \dots, M_t)$ and $(\wt{M}_0, \dots, \wt{M}_t)$ be two sequences of
  decisions, where $M_k$ and $\wt{M}_k \in \calX$. Let $h_t$ and $\tilde{h}_t$
  be the corresponding histories, and $(s_t, u_t)$ and $(\tilde{s}_t,
  \tilde{u}_t)$ be the corresponding state-control pairs at the end of round
  $t$. We have
  \begin{align*}
    \left\vert f_t(h_t) - f_t(\tilde{h}_t) \right\vert
    &= \left\vert c_t(s_t, u_t) - c_t(\tilde{s}_t, \tilde{u}_t) \right\vert \\
    &\leq L_0 D_\calX \max \{ \left\| s_t - \tilde{s}_t \right\|_2, \left\| u_t - \tilde{u}_t \right\|_2 \},
  \end{align*}
  where the last inequality follows from our assumptions about the functions
  $c_t$ and~\cref{lemma:appendix_olc_dx}. It suffices to bound the two norms on
  the right-hand side in terms of $\| h_t - \tilde{h}_t \|_\calH$. For $k = 0,
  \dots, t-1$, define $Z_k^{[s]} = \wt{F}^{t-k-1} G (M_k^{[s]} -
  \wt{M}_k^{[s]})$. Using~\cref{eq:appendix_olc_st}, we have
  \begin{align}
    \left\| s_t - \tilde{s}_t \right\|_2
    &= \left\| \sum_{k=0}^{t-1} \sum_{s=1}^{k+1} Z_k^{[s]} w_{k-s} \right\|_2 \nonumber \\ 
    &\leq \sum_{k=0}^{t-1} \sum_{s=1}^{k+1} \left\| Z_k^{[s]} w_{k-s} \right\|_2 \nonumber \\
    &= \sum_{k=0}^{t-1} \sum_{s=1}^{k+1} \left\| \rho^{-\frac{s}{2}} Z_k^{[s]} \rho^{\frac{s}{2}} w_{k-s} \right\|_2 \nonumber \\
    &\leq \sum_{k=0}^{t-1} \sum_{s=1}^{k+1} \left\| \rho^{-\frac{s}{2}} Z_k^{[s]} \right\|_2 \left\| \rho^{\frac{s}{2}} w_{k-s} \right\|_2 \nonumber \\
    &= \sum_{k=0}^{t-1} \sum_{s=1}^{k+1} \xi_{1+t-1-k} \left\| \rho^{-\frac{s}{2}} Z_k^{[s]} \right\|_2 \xi_{1+t-1-k}^{-1} \left\| \rho^{\frac{s}{2}} w_{k-s} \right\|_2 \nonumber \\
    &\leq \sqrt{\sum_{k=0}^{t-1} \sum_{s=1}^{k+1} \xi_{t-k}^2 \left\| \rho^{-\frac{s}{2}} Z_k^{[s]} \right\|_2^2 } \sqrt{\sum_{k=0}^{t-1} \sum_{s=1}^{k+1} \xi_{t-k}^{-2} \left\| \rho^{\frac{s}{2}} w_{k-s} \right\|_2^2 } \label{eq:appendix_olc_bound_l_1} \\
    &= \underbrace{\sqrt{\sum_{k=0}^{t-1} \xi_{t-k}^2 \sum_{s=1}^{k+1} \left\| \rho^{-\frac{s}{2}} Z_k^{[s]} \right\|_2^2 }}_{(a)} \underbrace{\sqrt{\sum_{k=0}^{t-1} \xi_{t-k}^{-2} \sum_{s=1}^{k+1} \left\| \rho^{\frac{s}{2}} w_{k-s} \right\|_2^2 }}_{(b)}, \nonumber
  \end{align}
  where~\cref{eq:appendix_olc_bound_l_1} follows from the Cauchy-Schwarz
  inequality. The specific choice of weighted norms on $\calX$ and $\calH$ allow
  us to bound the terms (a) and (b) in terms of $\| h_t - \tilde{h}_t \|_\calH$.
  We can bound the term (a) using the definition of $Z_k^{[s]}$, $\| \cdot
  \|_\calX$, and $\| \cdot \|_\calH$ as
  \begin{align}
    \sqrt{\sum_{k=0}^{t-1} \xi_{t-k}^2 \sum_{s=1}^{k+1} \left\| \rho^{-\frac{s}{2}} Z_k^{[s]} \right\|_2^2 }
    &= \sqrt{\sum_{k=0}^{t-1} \xi_{t-k}^2 \sum_{s=1}^{k+1} \rho^{-s} \left\| \wt{F}^{t-k-1} G (M_k^{[s]} - \wt{M}_k^{[s]}) \right\|_2^2 } \nonumber \\
    &\leq \sqrt{\sum_{k=0}^{t-1} \xi_{t-k}^2 \sum_{s=1}^{k+1} \rho^{-s} \left\| \wt{F}^{t-k-1} G (M_k^{[s]} - \wt{M}_k^{[s]}) \right\|_F^2 } \label{eq:appendix_olc_bound_l_3} \\
    &\leq \| h_t - \tilde{h}_t \|_\calH, \label{eq:appendix_olc_bound_l_4}
  \end{align}
  where~\cref{eq:appendix_olc_bound_l_3} follows from part 1
  of~\cref{lemma:matrix_norms} and~\cref{eq:appendix_olc_bound_l_4} follows from
  the definitions of $\| \cdot \|_\calX$ and $\| \cdot \|_\calH$. Using $\| w_t
  \|_2 \leq W$ for all rounds $t$, we can bound the term (b) as
  \begin{align}
    \sqrt{\sum_{k=0}^{t-1} \xi_{t-k}^{-2} \sum_{s=1}^{k+1} \left\| \rho^{\frac{s}{2}} w_{k-s} \right\|_2^2 }
    &\leq W \sqrt{\sum_{k=0}^{t-1} \xi_{t-k}^{-2} \sum_{s=1}^{k+1} \rho^s } \nonumber \\
    &\leq W \sqrt{\sum_{k=0}^{t-1} \xi_{t-k}^{-2} \frac{\rho (1 - \rho^{k+1})}{1 - \rho}} \nonumber \\
    &\leq W (1 - \rho)^{-1}, \label{eq:appendix_olc_bound_l_2}
  \end{align}
  where~\cref{eq:appendix_olc_bound_l_2} follows from the definition of
  $(\xi_k)$~(\cref{eq:appendix_olc_xi}).
  Substituting~\cref{eq:appendix_olc_bound_l_2,eq:appendix_olc_bound_l_4}
  in~\cref{eq:appendix_olc_bound_l_1}, we have
  \begin{equation*}
    \left\| s_t - \tilde{s}_t \right\|_2 \leq W (1 - \rho)^{-1} \| h_t - \tilde{h}_t \|_\calH.
  \end{equation*}

  Similarly,
  \begin{align*}
    \left\| u_t - \tilde{u}_t \right\|
    &= \left\| K (s_t - \tilde{s}_t) + \sum_{s=1}^{t+1} (M_t^{[s]} - \wt{M}_t^{[s]}) w_{t-s} \right\|_2 \\
    &\leq O \left( W \kappa (1 - \rho)^{-1} \left\| h_t - \tilde{h}_t \right\|_\calH \right),
  \end{align*}
  where the last inequality follows from our assumption that $\| K \|_2 \leq
  \kappa$ and the above inequality for $\| s_t - \tilde{s}_t \|_2$. This
  completes the proof.
\end{proof}

\begin{lemma}\label{lemma:appendix_olc_Lcirc}
  The Lipschitz constant of $\circfn{f}_t$ can be bounded above as
  \begin{equation*}
    \circfn{L} \leq O \left( L_0 D_\calX W \kappa^5 (1 - \rho)^{-\frac32} \right),
  \end{equation*}
  where $D_\calX$ is defined in~\cref{lemma:appendix_olc_dx}.
\end{lemma}

\begin{proof}
  Using ~\cref{lemma:appendix_olc_Ak} that bounds $\| A^k \|$, we have
  \begin{equation*}
    \sqrt{ \sum_{k=0}^\infty \left\| A^k \right\|^2 } \leq O \left( \kappa^4 (1 - \rho)^{-\frac12} \right).
  \end{equation*}
  Using~\cref{thm:lcirc} that bounds $\circfn{L}$ in terms of $L$ and the above,
  we have
  \begin{equation*}
    \circfn{L} \leq O \left( L \kappa^4 (1 - \rho)^{-\frac12} \right) \leq O \left( L_0 D_\calX W \kappa^5 (1 - \rho)^{-\frac32} \right),
  \end{equation*}
  where the last inequality follows from~\cref{lemma:appendix_olc_L}.
\end{proof}

Now we restate and prove~\cref{thm:regret_upper_bound_olc}.

\thmregretupperboundolc*

\begin{proof}
  Using~\cref{thm:regret_upper_bound} and the above lemmas, we can upper bound
  the policy regret of~\cref{alg:ftrl} for the online linear control problem by
  \begin{align*}
    & O \left( \sqrt{ \frac{D}{\alpha} T L \circfn{L} H_2} \right) \\
    &= O \left( \sqrt{ d \kappa^8 (1-\rho)^{-1} \  T \  \left( L_0 W^2 \kappa^9 (1-\rho)^{-3} \right)^2 \kappa^4 (1-\rho)^{-\frac12} \  \kappa^4 (1-\rho)^{-\frac32} } \right) \\
    &= O \left( L_0 W^2 \sqrt{T} d^\frac12 \kappa^{17} (1-\rho)^{-4.5} \right).
  \end{align*}
  This completes the proof.
\end{proof}


\subsection{Existing Regret Bound}\label{subsec:appendix_olc_existing}

The upper bound on policy regret for the online linear control problem in
existing work is given in~\citet[Theorem 5.1]{agarwalBHKS2019online}. The
theorem statement only shows the dependence on $\tilde{L}, W$, and $T$. The
dependence on $d, \kappa$, and $\rho$ can be found in the details of the proof.
Below we give a detailed accounting of all of these terms in their regret bound.

To simplify notation let $\gamma = 1 - \rho$.~\citet{agarwalBHKS2019online}
define
\begin{equation*}
  H = \frac{\kappa^2}{\gamma} \log(T)
  \quad
  \text{ and }
  \quad
  C = \frac{W (\kappa^2 + H \kappa_B \kappa^2 a)}{\gamma (1 - \kappa^2(1-\gamma)^{H+1})} + \frac{\kappa_B \kappa^3 W}{\gamma}.
\end{equation*}
The value of $a$ is not specified in Theorem 5.1. However, from Theorem 5.3 and
the definition of $\calM$ in Algorithm 1 their paper, we can infer that $a =
\kappa_B \kappa^3$.

The final regret bound is obtained by summing Equations 5.1, 5.3, and 5.4. Given
the definition of $H$ above, we have that
\begin{equation*}
  (1 - \gamma)^{H+1} \leq \exp(- \kappa^2 \log T) = T^{-\kappa^2}.
\end{equation*}
So, the dominant term in the regret bound is Equation 5.4, which is
\begin{equation*}
  O \left( L_0 W C d^{\frac{3}{2}} \kappa_B^2 \kappa^6 H^{2.5} \gamma^{-1} \sqrt{T} \right).
\end{equation*}

Substituting the values of $H$ and $C$ from above and collecting terms, we have
that the upper bound on policy regret in existing work~\citep[Theorem
5.1]{agarwalBHKS2019online} is
\begin{align*}
  & O \left( L_0 W d^{\frac{3}{2}} \sqrt{T} \log(T)^{2.5} \kappa_B^2 \kappa^{11} \gamma^{-3.5} C \right) \\
  &= O \left( L_0 W d^{\frac{3}{2}} \sqrt{T} \log(T)^{2.5} \kappa_B^2 \kappa^{11} \gamma^{-3.5} \left( \frac{W (\kappa^2 + H \kappa_B \kappa^2 a)}{\gamma (1 - \kappa^2(1-\gamma)^{H+1})} + \frac{\kappa_B \kappa^3 W}{\gamma} \right) \right) \\
  &= O \left( L_0 W d^{\frac{3}{2}} \sqrt{T} \log(T)^{2.5} \kappa_B^2 \kappa^{11} \gamma^{-3.5} \left( \frac{W \kappa^2}{\gamma (1 - \kappa^2(1-\gamma)^{H+1})} + \frac{W \kappa_B^2 \kappa^7 \log(T)}{\gamma^2 (1 - \kappa^2(1-\gamma)^{H+1})} + \frac{\kappa_B \kappa^3 W}{\gamma} \right) \right) \\
  &= O \left( L_0 W^2 d^{\frac{3}{2}} \sqrt{T} \log(T)^{2.5} \kappa_B^2 \kappa^{13} \gamma^{-4.5}(1 - \kappa^2(1-\gamma)^{H+1})^{-1} \right) \\
  &\quad + O \left( L_0 W^2 d^{\frac{3}{2}} \sqrt{T} \log(T)^{3.5} \kappa_B^4 \kappa^{18} \gamma^{-5.5} (1 - \kappa^2(1-\gamma)^{H+1})^{-1} \right) \\
  &\quad + O \left( L_0 W^2 d^{\frac{3}{2}} \sqrt{T} \log(T)^{2.5} \kappa_B^3 \kappa^{14} \gamma^{-4.5} \right) \\ 
  &=O \left( L_0 W^2 d^{\frac{3}{2}} \sqrt{T} \log(T)^{3.5} \kappa_B^4 \kappa^{18} \gamma^{-5.5} \right).
\end{align*}
Above we used that $\lim_{T \to \infty} (1 - \kappa^2(1-\gamma)^{H+1})^{-1} = 1$
to simplify the expressions. Therefore, the upper bound on policy regret for the
online linear control problem in existing work is
\begin{equation}\label{eq:olc_existing_regret_bound}
  O \left( L_0 W^2 d^{\frac{3}{2}} \sqrt{T} \log(T)^{3.5} \kappa_B^4 \kappa^{18} \gamma^{-5.5} \right).
\end{equation}

\section{Online Performative Prediction}
\label{sec:appendix_opp}


Before formulating the online performative prediction problem in our OCO with
unbounded memory framework, we state the definition of $1$-Wasserstein distance
that we use in our regret analysis. Informally, the $1$-Wasserstein distance is
a measure of the distance between two probability measures.

\begin{definition}[1-Wasserstein Distance]
  Let $(\calZ, d)$ be a metric space. Let $\mathbb{P}(\calZ)$ denote the set of
  Radon probability measures $\nu$ on $\calZ$ with finite first moment. That is,
  there exists $z' \in \calZ$ such that $\E_{z \sim \nu} [d(z, z')] < \infty$.
  The $1$-Wasserstein distance between two probability measures $\nu, \nu' \in
  \mathbb{P}(\calZ)$ is defined as
  \begin{equation*}
    W_1(\nu, \nu') = \sup \{ \E_{z \sim \nu} [f(z)] - \E_{z \sim \nu'} [f(z)]  \},
  \end{equation*}
  where the supremum is taken over all $1$-Lipschitz continuous functions $f :
  \calZ \to \R$.
\end{definition}

\subsection{Formulation as OCO with Unbounded Memory}\label{subsec:appendix_opp_formulation}

Now we formulate the online performative prediction problem in our framework by
defining the decision space $\calX$, the history space $\calH$, and the linear
operators $A : \calH \to \calH$ and $B : \calW \to \calH$. Then, we define the
functions $f_t : \calH \to \R$ in terms of $l_t$ and finally, prove an upper
bound on the policy regret. For notational convenience, let $(y_k)$ denote the
sequence $(y_0, y_1, \dots)$.

Let $\rho \in (0, 1)$. Let the decision space $\calX \subseteq \R^d$ be closed
and convex with $\| \cdot \|_\calX = \| \cdot \|_2$. Let the history space
$\calH$ be the $\ell^1$-direct sum of countably infinte number of copies of
$\calX$. Define the linear operators $A : \calH \to \calH$ and $B : \calX \to
\calH$ as
\begin{equation*}
  A((y_0, y_1, \dots)) = (0, \rho y_0, \rho y_1, \dots)
  \quad
  \text{ and }
  \quad
  B(x) = (x, 0, \dots).
\end{equation*}
Note that the problem is an OCO with $\rho$-discounted infinite memory problem
and follows linear sequence dynamics with the $1$-norm~(\cref{def:sequence}).

Given a sequence of decisions $(x_k)_{k=1}^t$, the history is $h_t = (x_t, \rho x_{t-1}, \dots, \rho^{t-1} x_1, 0, \dots)$ and the data distribution $p_t = p_t(h_t)$ satisfies:
\begin{equation}\label{eq:opp_dist_pt}
  z \sim p_t \text{ iff } z \sim \sum_{k=1}^{t-1} (1-\rho) \rho^{k-1} (\xi + F x_{t-k}) + \rho^t p_1. 
\end{equation}
This follows from the recursive definition of $p_t$ and parametric assumption about $\calD(x)$. Define the functions $f_t : \calH \to [0,1]$ by
\begin{equation*}
  f_t(h_t) = \E_{z \sim p_t} [ l_t(x_t, z) ]. 
\end{equation*}

With the above formulation and definition of $f_t$, the original goal of
minimizing the difference between the algorithm's total loss and the total loss
of the best fixed decision is equivalent to minimizing the policy regret,
\begin{equation*}
  \sum_{t=1}^T f_t(h_t) - \min_{x \in \calX} \sum_{t=1}^T \circfn{f}_t(x).
\end{equation*}


\subsection{Regret Analysis}\label{subsec:appendix_opp_regret_analysis}

\begin{lemma}\label{lemma:appendix_opp_Ak}
  The operator norm $\| A^s \|$ is bounded above as
  \begin{equation*}
    \| A^s \| \leq O \left( \rho^s \right).
  \end{equation*}
\end{lemma}

\begin{proof}
  Recall the definition of $\calH$ and $\| \cdot \|_\calH$. Let
  \begin{equation*}
    (y_0, y_1, \dots) = (x_0, \rho x_1, \rho^2 x_2, \dots)
  \end{equation*}
  be an element of $\calH$ with unit norm, i.e.,
  \begin{equation*}
    \sum_{k=0}^\infty \| y_k \| = 1.
  \end{equation*}
  From the definition of the operator $A$, we have
  \begin{equation*}
    A^s((y_0, y_1, \dots)) = (0, \dots, 0, \rho^s x_0, \rho^{s+1} x_1, \dots).
  \end{equation*}
  Now we bound $\| A^s \|$ as follows. By definition of $A^s$ and $\| \cdot
  \|_\calH$, we have
  \begin{equation*}
    \| A^s((y_0, y_1, \dots)) \|
    = \sum_{k=0}^\infty \rho^{s+k} \| x_k \|
    = \rho^s \sum_{k=0}^\infty \rho^k \| x_k \|
    = \rho^s \sum_{k=0}^\infty \| y_k \|
    = \rho^s. \qedhere
  \end{equation*}
\end{proof}

\begin{lemma}\label{lemma:appendix_opp_H2}
  The $1$-effective memory capacity is bounded above as
  \begin{equation*}
    H_2 \leq O \left( (1 - \rho)^{-2} \right).
  \end{equation*}
\end{lemma}

\begin{proof}
  Using~\cref{lemma:appendix_opp_Ak} to bound $\| A^k \|$, we have
  \begin{equation*}
    H_1 = \sum_{k=0}^\infty k \| A^k \| = \sum_{k=0}^\infty k \rho^k \leq O \left( (1 - \rho)^{-2} \right) . \qedhere
  \end{equation*}
\end{proof}

\begin{lemma}\label{lemma:appendix_opp_D}
  Suppose $R : \calX \to \R$ is defined by $R(x) = \frac12 \| x \|_\calX^2$.
  Then, it is $1$-strongly-convex and $D = \max_{x, \tilde{x} \in \calX} |R(x) -
  R(\tilde{x})| \leq D_\calX^2$.
\end{lemma}

\begin{proof}
  Note that $R$ is $1$-strongly-convex by definition. By the assumption that $\|
  x \|_\calX \leq D_\calX$ for all $x \in \calX$, we have that $D \leq
  D_\calX^2$.
\end{proof}

\begin{lemma}\label{lemma:appendix_opp_L}
  The Lipschitz constant of $f_t$ can be bounded above as
  \begin{equation*}
    L \leq O \left( L_0 \frac{1-\rho}{\rho} \| F \|_2 \right).
  \end{equation*}
\end{lemma}

\begin{proof}
Let $(x_1, \dots, x_t)$ and $(\tilde{x}_1, \dots, \tilde{x}_t)$ be two sequences
of decisions, where $x_k, \tilde{x}_k \in \calX$. Let $h_t$ and $\tilde{h}_t$ be
the corresponding histories, and $p_t$ and $\tilde{p}_t$ be the corresponding
distributions at the end of round $t$. We have
\begin{align*}
  & \left\vert f_t(h_t) - f_t(\tilde{h}_t) \right\vert \\
  &= \left\vert \E_{z \sim p_t} \left[ l_t(x_t, z) \right] - \E_{z \sim \tilde{p}_t} \left[ l_t(\tilde{x}_t, z) \right] \right\vert \\
  &= \left\vert \E_{z \sim p_t} \left[ l_t(x_t, z) \right] - \E_{z \sim p_t} \left[ l_t(\tilde{x}_t, z) \right] + \E_{z \sim p_t} \left[ l_t(\tilde{x}_t, z) \right] - \E_{z \sim \tilde{p}_t} \left[ l_t(\tilde{x}_t, z) \right] \right\vert \\
  &\leq L_0 \| x_t - \tilde{x}_t \|_2 + L_0 W_1 (p_t, \tilde{p}_t),
\end{align*}
where the last inequality follows from the assumptions about the functions $l_t$
and the definition of the Wasserstein distance $W_1$. By definition of
$p_t$~(\cref{eq:opp_dist_pt}), we have
\begin{align*}
  W_1(p_t, \tilde{p}_t)
  &\leq \sum_{k=1}^{t-1} \frac{1-\rho}{\rho} \rho^k \| F \|_2 \| x_{t-k} - \tilde{x}_{t-k} \|_2 \\
  &\leq \frac{1-\rho}{\rho} \| F \|_2 \| h_t - \tilde{h}_t \|_\calH,
\end{align*}
where the last inequality follows from the definition of $\| \cdot \|_\calH$.
Therefore, $L \leq L_0 \frac{1-\rho}{\rho} \| F \|_2$.
\end{proof}

\begin{lemma}\label{lemma:appendix_opp_Lcirc}
  The Lipschitz constant of $f_t$ can be bounded above as
  \begin{equation*}
    \circfn{L} \leq O \left( L_0 \frac{1}{\rho} \| F \|_2 \right).
  \end{equation*}
\end{lemma}

\begin{proof}
  Using~\cref{lemma:appendix_opp_Ak} that bounds $\| A^k \|$, we have
  \begin{equation*}
    \sum_{k=0}^\infty \| A^k \| = (1-\rho)^{-1}.
  \end{equation*}
  Using~\cref{thm:lcirc} that bounds $\circfn{L}$ in terms of $L$ and the above,
  we have
  \begin{equation*}
    \circfn{L} \leq O \left( L (1-\rho)^{-1} \right) = O \left( L_0 \frac{1}{\rho} \| F \|_2 \right),
  \end{equation*}
  where the last equality follows from~\cref{lemma:appendix_opp_L}.
\end{proof}

Now we restate and prove~\cref{thm:regret_upper_bound_opp}.

\thmregretupperboundopp*

\begin{proof}
  Using~\cref{thm:regret_upper_bound} and the above lemmas, we can upper bound
  the policy regret of~\cref{alg:ftrl} for the online performative prediction
  problem by
  \begin{equation*}
    O \left( \sqrt{ \frac{D}{\alpha} T L \circfn{L} H_1 } \right)
    = O \left( D_\calX L_0 \| F \|_2 (1-\rho)^{-\frac12} \rho^{-1} \sqrt{T} \right).
  \end{equation*}
  This completes the proof.
\end{proof}

We note that the upper bound can be improved by defining a weighted norm on
$\calH$ similar to the approach in~\cref{sec:appendix_olc}. However, here we
present the looser anaysis for simplicity of exposition.

\section{Implementation Details for~\cref{alg:ftrl}}
\label{sec:appendix_implementation}


In this section we discuss how to implement~\cref{alg:ftrl} efficiently.

\paragraph{Dimensionality of $\calX$.}
First, note that the decisions $x \in \calX$ could be high-dimensional, e.g., an
unbounded sequence of matrices as in the online linear control problem, but this
is external to our framework and is application dependent. Our framework can be
applied to $\calX$ or to a lower-dimensional decision space $\calX'$. However,
the choice of $\calX'$ and analyzing the difference
\begin{equation*}
  \min_{x' \in \calX'} \sum_{t=1}^{T} \circfn{f}_t(x') - \min_{x \in \calX} \sum_{t=1}^{T} \circfn{f}_t(x)
\end{equation*}
is application dependent. For example, for the online linear control problem one
could consider a restricted class of disturbance-action controllers that operate
on a constant number of past disturbances as opposed to all the past
disturbances, and then analyze the difference between these two policy classes.
See, for example,~\citet[Lemma 5.2]{agarwalBHKS2019online}. 

\paragraph{Computational cost of each iteration of~\cref{alg:ftrl}.}
Now we discuss how to implement each iteration of~\cref{alg:ftrl} efficiently.
We are interested in the computational cost of computing the decision $x_{t+1}$
as a function of $t$. (Given the above discussion about the dimensionality of
$\calX$, we ignore the fact that the dimensionality of the decisions themselves
could depend on $t$.) Therefore, for the purposes of this section we (i) use
$O(\cdot)$ notation to hide absolute constants and problem parameters excluding
$t$ and $T$; (ii) invoke the operators $A$ and $B$ by calling oracles
$\calO_A(\cdot)$ and $\calO_B(\cdot)$; and (iii) evaluate the functions $f_t$ by
calling oracles $\calO_f(t, \cdot)$. Recall from Assumption~\ref{ass:feedback}
that we assume the learner knows the operators $A$ and $B$, and observes $f_t$
at the end of each round $t$. So, the oracles $\calO_A, \calO_B$, and $\calO_f$
are readily available.

\cref{alg:ftrl} chooses the decision $x_{t+1}$ as
\begin{equation*}
  x_{t+1}
  \in \argmin_{x \in \calX} \sum_{s=1}^t \circfn{f}_s(x) + \frac{R(x)}{\eta}
  = \argmin_{x \in \calX} \underbrace{\sum_{s=1}^t f_s \left( \sum_{k=0}^{s-1} A^k B x \right)}_{= F_t(x)} + \frac{R(x)}{\eta}.
\end{equation*}
Since $F_t(x)$ is a sum of $f_1, \dots, f_t$, evaluating $F_t(x)$ requires
$\Theta(t)$ oracle calls to $\calO_f$. However, this issue is present in FTRL
for OCO and OCO with finite memory as well and is not specific to our framework.
To deal with this issue, one could consider mini-batching
algorithms~\citep{dekelTA2012online,altschulerT2018online,chenYLK2020minimax}
such as~\cref{alg:minibatch_ftrl}.

A na\"ive implementation to evaluate $F_t(x)$ could require $O(t^3)$ oracle
calls to $\calO_A$: for each $s \in [t]$, constructing the argument
$\sum_{k=0}^{s-1} A^k B x$ for $f_s$ could require $k$ oracle calls to $\calO_A$
to compute $A^k B x$, for a total of $O(s^2)$ oracle calls. However, $F_t(x)$
can be evaluated with just $O(t)$ oracle calls to $\calO_A$ by constructing the
arguments incrementally. For $t \geq 0$, define $\Gamma_t : \calX \to \calH$ as
\begin{align*}
  \Gamma_0(x) &= B x \\
  \Gamma_t(x) &= A \left( \Gamma_{t-1}(x) \right) \quad \text{ for } t \geq 1.
\end{align*}
Note that $\Gamma_t(B x) = A^t B x$. Also, for $t \geq 1$, define $\Phi_t :
\calX \to \calH$ as
\begin{align*}
  \Phi_1(x) &= \Gamma_0(x) \\
  \Phi_t(x) &= \Phi_{t-1}(x) + \Gamma_{t-1}(x) \quad \text{ for } t \geq 2.
\end{align*}
Note that $\Phi_s(x) = \sum_{k=0}^{s-1} A^k B x$ is the argument for $f_s$.
These can be constructed incrementally as follows.
\begin{enumerate}
  \item Construct $\Gamma_0(x)$ using one oracle call to $\calO_B$.
  \item For $s = 1$,
  \begin{enumerate}
    \item Construct $\Phi_1(x) = \Gamma_0(x)$.
    \item Construct $\Gamma_1(x)$ from $\Gamma_0(x)$ using one oracle call to
    $\calO_A$. 
  \end{enumerate}
  \item For $s \geq 2$,
  \begin{enumerate}
    \item Construct $\Phi_s(x)$ by adding $\Phi_{s-1}(x)$ and $\Gamma_{s-1}(x)$.
    This can be done in $O(1)$ time. Recall from our earlier discussion that
    $O(\cdot)$ hides absolute constants and problem parameters excluding $t$ and
    $T$.
    \item Construct $\Gamma_s(x)$ from $\Gamma_{s-1}(x)$ using one oracle call
    to $\calO_A$.
  \end{enumerate}
\end{enumerate}
By incrementally constructing $\Phi_s(x)$ as above, we can evaluate $F_t(x)$ in
$O(t)$ time with $O(1)$ oracle calls to $\calO_B$, $O(t)$ oracle calls to
$\calO_A$, and $O(t)$ oracle calls to $\calO_f$.

\paragraph{Memory usage of~\cref{alg:ftrl}.}
We end with a brief discussion of the memory usage of~\cref{alg:ftrl}.  We are
interested in the memory usage of computing the decision $x_{t+1}$ as a function
of $t$. (Given the discussion about the dimensionality of $\calX$ at the start
of this section, we ignore the fact that the dimensionality of the decisions
themselves could depend on $t$.) For each $t \in [T]$, the memory usage could be
as low as $O(1)$ (if, for example, $\calX \subseteq \R^d$, and $A, B \in \R^{d
\times d}$, which implies that $\Phi_t(x)$ is a $d$-dimensional vector) or as
high as $O(t)$ (if, for example, $\Phi_t(x)$ is a $t$-length sequence of
$d$-dimensional vectors).  However, the memory usage is already $\Omega(t)$ to
store the functions $f_1, \dots, f_t$. Therefore,~\cref{alg:ftrl} only incurs a
constant factor overhead.

\section{An Algorithm with A Low Number of Switches: Mini-Batch FTRL}
\label{sec:appendix_minibatchftrl}


In this section we present an algorithm~(\cref{alg:minibatch_ftrl}) for OCO with
unbounded memory that provides the same upper bound on policy regret
as~\cref{alg:ftrl} while guaranteeting a small number of
switches.~\cref{alg:minibatch_ftrl} combines FTRL on the functions
$\circfn{f}_t$ with a mini-batching approach. First, it divides rounds into
batches of size $S$, where $S$ is a parameter. Second, at the start of batch $b
\in \{1, \dots, \ceil{\nicefrac{T}{S}} \}$, it performs FTRL on the functions
$\{ g_1, \dots, g_b \}$, where $g_i$ is the average of the functions
$\circfn{f}_t$ in batch $i$. Then, it uses this decision for the entirety of the
current batch. By design,~\cref{alg:minibatch_ftrl} switches decisions at most
$O(\nicefrac{T}{S})$ times.  This algorithm is insipired by similar algorithms
for online learning and
OCO~\citep{dekelTA2012online,altschulerT2018online,chenYLK2020minimax}.

\begin{algorithm}[ht]
  \caption{\texttt{Mini-Batch FTRL}}
  \label{alg:minibatch_ftrl}
  \DontPrintSemicolon
  \SetKwInOut{Input}{Input}
  \SetKwInOut{Output}{Output}

  \Input{ Time horizon $T$, step size $\eta$, $\alpha$-strongly-convex regularizer $R : \calX \to \R$, batch size $S$.}

  Initialize history $h_0 = 0$. \\

  \For{$t = 1, 2, \dots, T$}{
    \If{$t \mod S = 1$}{
      Let $N_t = \{ 1, \dots, \ceil{\frac{t}{S}} \}$ denote the number of batches so far. \\
      For $b \in N_t$, let $T_b = \{ (b-1)S + 1, \dots, b S \}$ denote the rounds in batch $b$. \\
      For $b \in N_t$, let $g_b = \frac{1}{S} \sum_{s \in T_b} \circfn{f}_s$. denote the average of the functions in batch $b$. \\
      Learner chooses $x_t \in \argmin_{x \in \calX} \sum_{b \in N_t} g_b(x) + \frac{R(x)}{\eta}$.
    }
    \Else{
      Learner chooses $x_t = x_{t-1}$.
    }
    Set $h_t = A h_{t-1} + B x_t$. \\
    Learner suffers loss $f_t(h_t)$ and observes $f_t$.
  }
\end{algorithm}

\begin{restatable}{theorem}{thmregretupperboundminibatch}\label{thm:regret_upper_bound_minibatch}
  Consider an online convex optimization with unbounded memory problem specified
  by $(\calX, \calH, A, B)$. Let the regularizer $R : \calX \to \R$ be
  $\alpha$-strongly-convex and satisfy $|R(x) - R(\tilde{x})| \leq D$ for all
  $x, \tilde{x} \in \calX$.~\cref{alg:minibatch_ftrl} with batch size $S$ and
  step-size $\eta$ satisfies
  \begin{equation*}
    R_T(\texttt{Mini-Batch FTRL}) \leq \frac{S D}{\eta} + \eta \frac{T \circfn{L}^2}{\alpha} + \eta \frac{T L \circfn{L} H_1}{S \alpha}.
  \end{equation*}
  If $\eta = \sqrt{\frac{\alpha S D}{T \circfn{L} \left( \frac{L H_1}{S} +
  \circfn{L} \right)}}$, then
  \begin{equation*}
    R_T(\texttt{Mini-Batch FTRL}) \leq O \left( \sqrt{\frac{D}{\alpha} T \left( L \circfn{L} H_1 + S \circfn{L}^2 \right)} \right).
  \end{equation*}
\end{restatable}

Setting the batch size to be $S = \nicefrac{L H_1}{\circfn{L}}$ we obtain the
same upper bound on policy regret as~\cref{alg:ftrl} while guaranteeing that the
decisions $x_t$ switch at most $\nicefrac{T \circfn{L}}{L H_1}$ times.

\begin{corollary}
  Consider an online convex optimization with unbounded memory problem specified
  by $(\calX, \calH, A, B)$. Let the regularizer $R : \calX \to \R$ be
  $\alpha$-strongly-convex and satisfy $|R(x) - R(\tilde{x})| \leq D$ for all
  $x, \tilde{x} \in \calX$.~\cref{alg:minibatch_ftrl} with batch size $S =
  \frac{L H_1}{\circfn{L}}$ and step-size $\eta = \sqrt{\frac{\alpha S D}{T
  \circfn{L} \left( \frac{L H_1}{S} + \circfn{L} \right)}}$ satisfies
  \begin{equation*}
    R_T(\texttt{Mini-Batch FTRL}) \leq O \left( \sqrt{\frac{D}{\alpha} T L \circfn{L} H_1} \right).
  \end{equation*}
  Furthermore, the decisions $x_t$ switch at most $\frac{T \circfn{L}}{L H_1}$
  times.
\end{corollary}

Intuitively, in the OCO with unbounded memory framework each decision $x_t$ is
penalized not just in round $t$ but in future rounds as well. Therefore, instead
of immediately changing the decision, it is prudent to stick to it for a while,
collect more data, and then switch decisions. For the OCO with finite memory
problem, the constant memory length $m$ provides a natural measure of how long
decisions penalized for and when one should switch decisions. In the general
case, this is measured by the quantity $\nicefrac{L H_1}{\circfn{L}}$. Note that
this simplifies to $m$ for OCO with finite memory for all $p$-norms.

\begin{proof}[Proof of~\cref{thm:regret_upper_bound_minibatch}]
  For simplicity, assume that $T$ is a multiple of $S$. Otherwise, the same
  proof works after replacing $\frac{T}{S}$ with $\ceil{\frac{T}{S}}$. Let $x^*
  \in \argmin_{x \in \calX} \sum_{t=1}^T \circfn{f}_t(x)$. Note that we
  can write the regret as
  \begin{align*}
    R_T(\texttt{Mini-Batch FTRL})
    &= \sum_{t=1}^T f_t(h_t) - \min_{x \in \calX} \sum_{t=1}^T \circfn{f}_t(x) \\
    &= \underbrace{\sum_{t=1}^T f_t(h_t) - \circfn{f}_t(x_t)}_{(a)} + \underbrace{\sum_{t=1}^T \circfn{f}_t(x_t) - \circfn{f}_t(x^*)}_{(b)}.
  \end{align*}
  We can bound the term (b) using~\cref{thm:oco_ftrl} for
  mini-batches~\citep{dekelTA2012online,altschulerT2018online,chenYLK2020minimax}
  by
  \begin{equation*}
    \frac{S D}{\eta} + \eta \frac{T \circfn{L}^2}{\alpha}.
  \end{equation*}
  It remains to bound term (a). Let $N = \nicefrac{T}{S}$ denote the number of
  batches and $T_n = \{ (n - 1) S + 1, \dots, n S \}$ denote the rounds in batch
  $n \in [N]$. We can write
  \begin{align*}
    \sum_{t=1}^T f_t(h_t) - \circfn{f}_t(x_t)
    &= \sum_{t=1}^T f_t\left( \sum_{k=0}^{t-1} A^k B x_{t-k} \right) - f_t \left( \sum_{k=0}^{t-1} A^k B x_t \right) & \text{ by~\cref{def:circfn}}\\
    &\leq L \sum_{t=1}^T \left\| \sum_{k=0}^{t-1} A^k B x_{t-k} - \sum_{k=0}^{t-1} A^k B x_t  \right\| & \text{ by Assumption~\ref{ass:lipschitz}} \\
    &\leq \frac{T}{S} L \underbrace{\sum_{t \in T_N} \left\| \sum_{k=0}^{t-1} A^k B x_{t-k} - \sum_{k=0}^{t-1} A^k B x_t  \right\|}_{(c)},
  \end{align*}
  where the last inequality follows because of the following. Consider rounds
  $t_1 = b_1 S + r$ and $t_2 = b_2 S + r$ for $b_1 < b_2$ and $r \in [S]$. Then,
  $\| h_{t_1} - \sum_{k=0}^{t_1 - 1} A^k B x_{t_1} \| \leq \| h_{t_2} -
  \sum_{k=0}^{t_2 - 1} A^k B x_{t_2} \|$ because the latter sums over more terms
  in its history and decisions in consecutive batches have distance bounded
  above by $\nicefrac{\eta \circfn{L}}{\alpha}$~(\cref{thm:oco_ftrl}).
  Therefore, it suffices to show that term (c) is upper bounded by
  $\nicefrac{\eta \circfn{L} H_1}{\alpha}$. We have
  \begin{align*}
    \sum_{t \in T_N} \left\| \sum_{k=0}^{t-1} A^k B x_{t-k} - \sum_{k=0}^{t-1} A^k B x_t  \right\|
    &\leq \sum_{t \in T_N} \sum_{k=0}^{t-1} \left\| A^k B x_{t-k} - A^k B x_t  \right\| \\
    &\leq \sum_{t \in T_N} \sum_{k=0}^{t-1} \| A^k \| \| B \| \| x_{t-k} - x_t \| \\
    &\leq \sum_{t \in T_N} \sum_{k=0}^{t-1} \| A^k \| \| x_{t-k} - x_t \| & \text{ by Assumption~\ref{ass:norm_B}}.
    \intertext{Since the same decision $x_n$ is chosen in all rounds of batch $n$, we can reindex and rewrite}
    \sum_{t \in T_N} \left\| \sum_{k=0}^{t-1} A^k B x_{t-k} - \sum_{k=0}^{t-1} A^k B x_t  \right\|
    &\leq \sum_{t \in T_N} \sum_{k=0}^{t-1} \| A^k \| \| x_{t-k} - x_t \| \\
    &\leq \sum_{o=0}^{S-1} \sum_{n=1}^{N-1} \sum_{s=1}^{S} \| A^{(N-n-1)S+s+o} \| \| x_N - x_n \| \\
    &\leq \eta \frac{\circfn{L}}{\alpha} \sum_{o=0}^{S-1} \sum_{n=1}^{N-1} \sum_{s=1}^{S} (N-n) \| A^{(N-n-1)S+s+o} \| \\
    &= \eta \frac{\circfn{L}}{\alpha} \sum_{o=0}^{S-1} \sum_{n=1}^{N-1} \sum_{s=1}^{S} n \| A^{(n-1)S+s+o} \|,
  \end{align*}
  where the last inequality follows from bounding the distance between decision
  in consecutive batches~\cref{thm:oco_ftrl} and the triangle inequality.
  Expanding the triple sum yields
  \begin{align*}
    & \sum_{o=0}^{S-1} \sum_{n=1}^{N-1} \sum_{s=1}^{S} n \| A^{(n-1)S+s+o} \| \\
    &\leq \| A \| + \dots + \| A^S \| + 2 \| A^{S+1} \| + \dots + 2 \| A^{2S} \| + 3 \| A^{2S+1} \| + \dots + 3 \| A^{3S} \| + \dots \\
    &\quad + \| A^2 \| + \dots + \| A^{S+1} \| + 2 \| A^{S+2} \| + \dots + 2 \| A^{2S+1} \| + 3 \| A^{2S+2} \| + \dots + 3 \| A^{3S+1} \| + \dots \\
    &\quad \vdots \\
    &\quad + \| A^S \| + \dots + \| A^{2S-1} \| + 2 \| A^{2S} \| + \dots + 2 \| A^{3S-1} \| + 3 \| A^{3S} \| + \dots + 3 \| A^{4S-1} \| + \dots,
  \end{align*}
  where each line above corresponds to a value of $o \in \{ 0, \dots, S-1 \}$.
  Adding up these terms yields $H_1$. This completes the proof.
\end{proof}

Note that~\cref{thm:regret_upper_bound_minibatch} only provides an upper bound
on the policy regret for the general case. Unlike~\cref{alg:ftrl}, it is unclear
how to obtain a stronger bound depending on $H_p$ for the case of linear
sequence dynamics with the $\xi$-weighted $p$-norm for $p > 1$. The above proof
can be specialized for this special case, similar to the proofs
of~\cref{thm:lcirc,lemma:dist_actual_ideal_history}, to obtain
\begin{equation*}
  \sum_{t \in T_N} \left\| \sum_{k=0}^{t-1} A^k B x_{t-k} - \sum_{k=0}^{t-1} A^k B x_t  \right\|
  \leq \eta \frac{\circfn{L}}{\alpha} \sum_{o=0}^{S-1} \left( \sum_{n=1}^{N-1} \sum_{s=1}^{S} \left( n \| A^{(n-1)S+s+o} \| \right)^p \right)^{\frac{1}{p}} 
\end{equation*}
and
\begin{align*}
  & \sum_{o=0}^{S-1} \left( \sum_{n=1}^{N-1} \sum_{s=1}^{S} \left( n \| A^{(n-1)S+s+o} \| \right)^p \right)^{\frac{1}{p}} \\
  &\leq \left( \| A \|^p + \dots + \| A^S \|^p + 2^p \| A^{S+1} \|^p + \dots 2^p \| A^{2S} \|^p + 3^p \| A^{2S+1} \|^p + \dots \right)^{\frac{1}{p}} \\
  &\quad + \left( \| A^2 \|^p + \dots + \| A^{S+1} \|^p + 2^p \| A^{S+2} \|^p + \dots 2^p \| A^{2S+1} \|^p + 3^p \| A^{2S+2} \|^p + \dots \right)^{\frac{1}{p}} \\
  &\quad \vdots \\
  &\quad \left( \| A^S \|^p + \dots + \| A^{2S-1} \|^p + 2^p \| A^{2S} \|^p + \dots 2^p \| A^{3S-1} \|^p + 3^p \| A^{3S} \|^p + \dots \right)^{\frac{1}{p}}.
\end{align*}
The above expression cannot be easily simplified to $O(H_p)$. However, for the
special case of OCO with finite memory, which follows linear sequence dynamics
with the $2$-norm, we can do so by leveraging the special structure of the
linear operator $A_{\text{finite},m}$.

\begin{restatable}{theorem}{thmregretupperboundminibatchfinite}\label{thm:regret_upper_bound_minibatch_finite}
  Consider an online convex optimization with finite memory problem with
  constant memory length $m$ specified by $(\calX, \calH = \calX^m,
  A_{\text{finite},m}, B_{\text{finite},m})$. Let the regularizer $R : \calX \to
  \R$ be $\alpha$-strongly-convex and satisfy $|R(x) - R(\tilde{x})| \leq D$ for
  all $x, \tilde{x} \in \calX$.~\cref{alg:minibatch_ftrl} with batch size $m$
  and step-size $\eta = \sqrt{\frac{\alpha m D}{T \circfn{L} \left( L m^\frac12
  + \circfn{L} \right)}}$ satisfies
  \begin{equation*}
    R_T(\texttt{Mini-Batch FTRL}) \leq O \left( \sqrt{\frac{D}{\alpha} T L \circfn{L} m^{\frac{3}{2}}} \right) \leq O \left( m \sqrt{\frac{D}{\alpha} T L^2} \right).
  \end{equation*}
  Furthermore, the decisions $x_t$ switch at most $\frac{T}{m}$ times.
\end{restatable}

\begin{proof}
  Given the proof of~\cref{thm:regret_upper_bound_minibatch} and the above
  discussion, it suffices to show that
  \begin{equation*}
  \sum_{o=0}^{S-1} \left( \sum_{n=1}^{N-1} \sum_{s=1}^{S} \left( n \| A^{(n-1)S+s+o} \| \right)^2 \right)^{\frac{1}{2}} \leq H_2 = m^\frac32.
  \end{equation*}
  Recall that $\| A_{\text{finite}}^k \| = 1$ if $k \leq m$ and $0$ otherwise.
  Using this and $S = m$, we have that the above sum is at most $\sqrt{m} +
  \sqrt{m-1} + \dots + \sqrt{1} = O \left( m^\frac32 \right)$. This completes
  the proof.
\end{proof}


\section{Experiments}
\label{sec:appendix_experiments}

In this section we present some simple simulation experiments.\footnote{https://github.com/raunakkmr/oco-with-memory-code.}


\paragraph{Problem Setup.}
We consider the problem of online linear control with a constant input
controller class $\Pi = \{ \pi_u : \pi(s) = u \in \calU \}$. Let $T$ denote the
time horizon. Let $\calS = \R^d$ and $\calU = \{ u \in \R^d : \| u \|_2 \leq 1
\}$ denote the state and control spaces. Let $s_t$ and $u_t$ denote the state
and control at time $t$ with $s_0$ being the initial state. The system evolves
according to linear dynamics $s_{t+1} = F s_t + G u_t + w_t$, where $F, G \in
\R^{d \times d}$ are system matrices and $w_t \in \R^d$ is a disturbance. The
loss function in round $t$ is simply $c_t(s_t, u_t) = c_t(s_t) = \sum_{j=1}^d
s_{t,j}$, where $s_{t,j}$ denotes the $j$-th coordinate of $s_t$. The goal is to
choose a sequence of control inputs $u_0, \dots, u_{T-1} \in \calU$ to minimize
the regret
\begin{equation*}
  \sum_{t=0}^{T-1} c_t(s_t, u_t) - \min_{u \in \calU} \sum_{t=0}^{T-1} c_t(s_t^u, u),
\end{equation*}
where $s_t^u$ denotes the state in round $t$ upon choosing control input $u$ in
each round. Note that the state in round $t$ can be written as
\begin{equation*}
  s_t = \sum_{k=1}^t F^k G u_{t-k} + \sum_{k=1}^t F^k w_{t-k}.
\end{equation*}
Therefore, we can formulate this problem as an OCO with unbounded memory problem
by setting $\calX = \calU, \calH = \{ y \in \R^d : y = \sum_{k=0}^t F^k G u
\text{ for some } u \in \calU \text{ and } t \in \N \}, A(h) = F h, B(x) = G x$,
and $f_t(h_t) = c_t(\sum_{k=1}^{t} F^k G u_{t-k} + \sum_{k=1}^{t} F^k w_{t-k})$.
Note that $\calH, A$, and $B$ are all finite-dimensional.

\paragraph{Data.}
We set the time horizon $T = 750$ and dimension $d = 2$. We sample the
disturbances $\{ w_t \}$ from a standard normal distribution. We set the system
matrix $G$ to be the identity and the system matrix $F$ to be a diagonal plus
upper triangular matrix with the diagonal entries equal to $\rho$ and the upper
triangular entries equal to $\alpha$. We run simulations with various values of
$\rho$ and $\alpha$.

\paragraph{Implementation.}
We use the \texttt{cvxpy}
library~\citep{diamondB2016cvxpy,agrawalVDB2018rewriting} for
implementing~\cref{alg:ftrl}. We use step-sizes according
to~\cref{thm:regret_upper_bound,thm:regret_upper_bound_finite}. We run the
experiments on a standard laptop.

\paragraph{Results.}
We compare the regret with respect to the optimal control input of OCO with
unbounded memory and OCO with finite memory for various memory lengths $m$
in~\cref{fig:results_olc_rho_0.90} for $\rho = 0.90$
and~\cref{fig:results_olc_rho_0.95} for $\rho = 0.95$. There are a few important
takeaways.
\begin{enumerate}

  \item OCO with unbounded memory either performs as well as or better than OCO
  with finite memory, and it does so at comparable computational
  cost~(\cref{sec:appendix_implementation}). In fact, the regret curve for OCO
  with unbounded memory reaches an asymptote whereas this is not the case for
  OCO with finite memory for a variety of memory lengths.
  
  \item Knowledge of the spectral radius of $F$, $\rho$, is not sufficient to
  tune the memory length $m$ for OCO with finite memory. This is illustrated by
  comparing~\cref{fig:d_2_rho_0.9_ut_0.05,fig:d_2_rho_0.9_ut_0.07,fig:d_2_rho_0.9_ut_0.10,fig:d_2_rho_0.9_ut_0.12}.
  Even though small memory lengths perform well when the upper triangular value
  is small, they perform poorly when the upper triangular value is large. In
  contrast, OCO with unbounded memory performs well in all cases.

  \item For a fixed memory length, OCO with unbounded memory eventually performs better than OCO with finite memory. This is illustrated by
  comparing~\cref{fig:d_2_rho_0.9_ut_0.05,fig:d_2_rho_0.9_ut_0.07,fig:d_2_rho_0.9_ut_0.10,fig:d_2_rho_0.9_ut_0.12}.

  \item As we increase the memory length, the performance of OCO with finite
  memory eventually approaches that of OCO with unbounded memory. However, an
  advantage of OCO with unbounded memory is that it does not require tuning the
  memory length. For example, when $\rho = 0.90$ and the upper triangular entry
  of $F = 0.10$, OCO with finite memory with $m = 4$ performs comparably to $m =
  8$ and $m = 16$~(\cref{fig:d_2_rho_0.9_ut_0.10}). However, when the upper
  triangular entry of $F = 0.12$, then it performs much
  worse~(\cref{fig:d_2_rho_0.9_ut_0.12}).  However, OCO with unbounded memory
  performs well in all cases without the need for tuning an additional
  hyperparameter in the form of memory length.

\end{enumerate}

\begin{figure}[hb]
  \centering
  \begin{subfigure}[b]{0.4\textwidth}
      \centering
      \includegraphics[width=\textwidth]{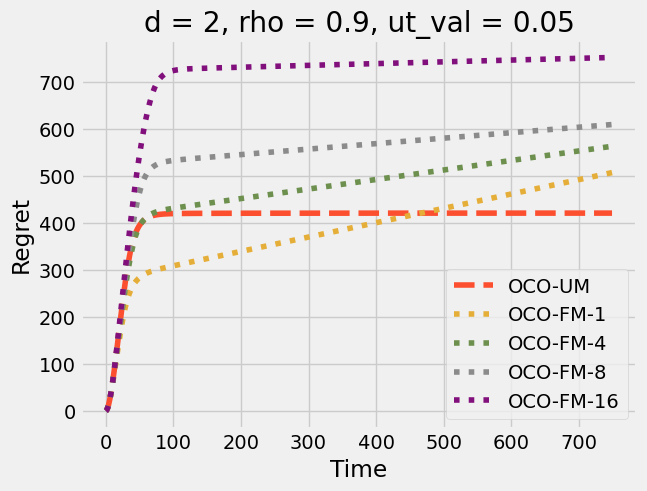}
      \caption{}
      \label{fig:d_2_rho_0.9_ut_0.05}
  \end{subfigure}
  \hfill
  \begin{subfigure}[b]{0.4\textwidth}
      \centering
      \includegraphics[width=\textwidth]{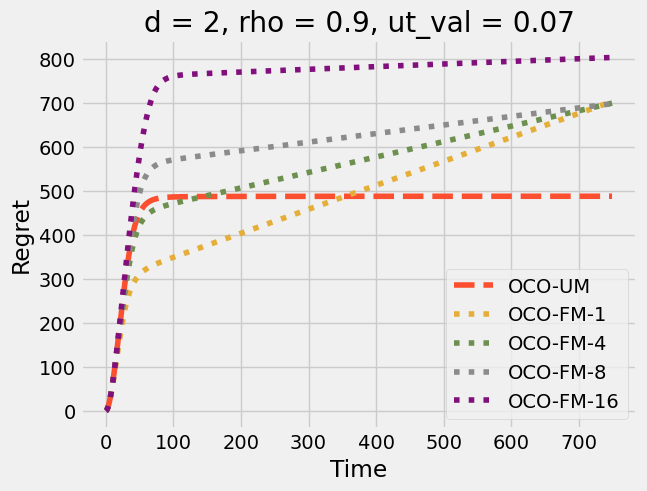}
      \caption{}
      \label{fig:d_2_rho_0.9_ut_0.07}
  \end{subfigure}
  \hfill
  \begin{subfigure}[b]{0.4\textwidth}
      \centering
      \includegraphics[width=\textwidth]{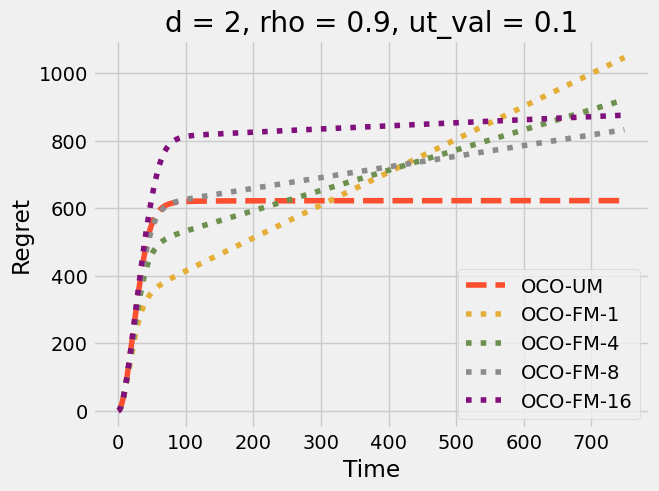}
      \caption{}
      \label{fig:d_2_rho_0.9_ut_0.10}
  \end{subfigure}
  \hfill
  \begin{subfigure}[b]{0.4\textwidth}
      \centering
      \includegraphics[width=\textwidth]{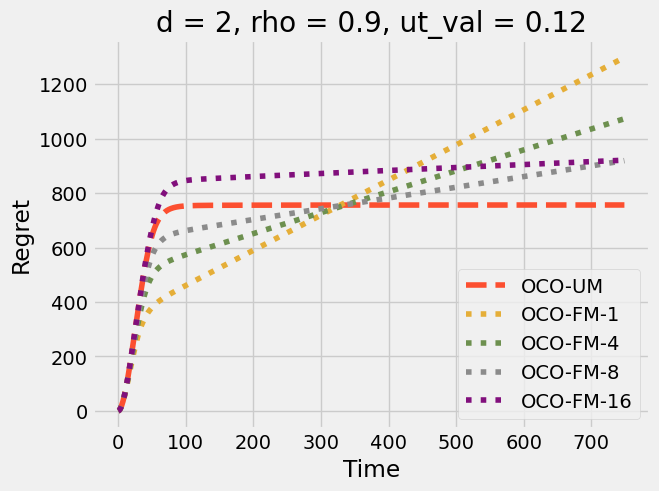}
      \caption{}
      \label{fig:d_2_rho_0.9_ut_0.12}
  \end{subfigure}
  \caption{Regret plot for $\rho = 0.90$. The label \texttt{OCO-UM} refers to formulating the problem as an OCO with unbounded memory problem. The \texttt{OCO-FM-m} refers to formulating the problem as an OCO with finite memory problem with constant memory length $m$. The titles of the plots indicate the values of the dimension, the diagonal entries of $F$, and the upper triangular entries of $F$.}
  \label{fig:results_olc_rho_0.90}
\end{figure}

\begin{figure}[hb]
  \centering
  \begin{subfigure}[b]{0.4\textwidth}
      \centering
      \includegraphics[width=\textwidth]{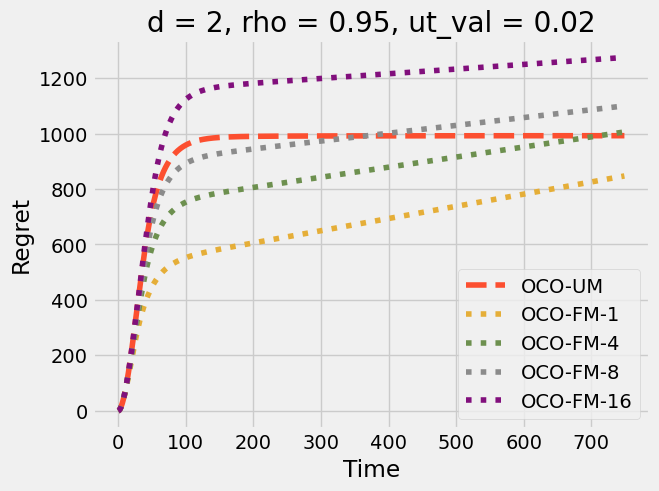}
      \caption{}
      \label{fig:d_2_rho_0.95_ut_0.02}
  \end{subfigure}
  \hfill
  \begin{subfigure}[b]{0.4\textwidth}
      \centering
      \includegraphics[width=\textwidth]{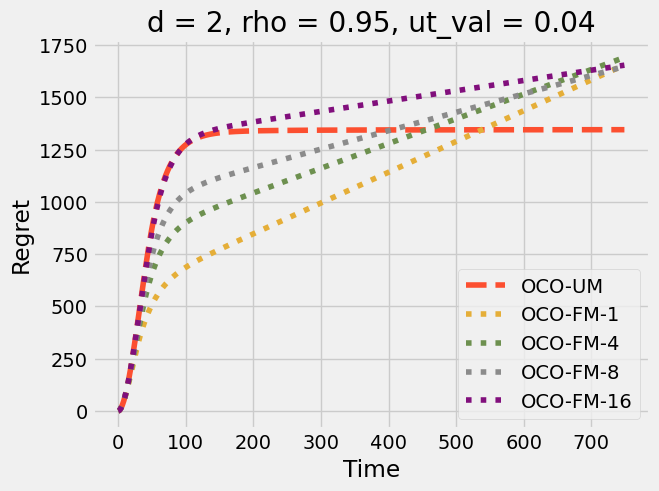}
      \caption{}
      \label{fig:d_2_rho_0.95_ut_0.04}
  \end{subfigure}
  \hfill
  \begin{subfigure}[b]{0.4\textwidth}
      \centering
      \includegraphics[width=\textwidth]{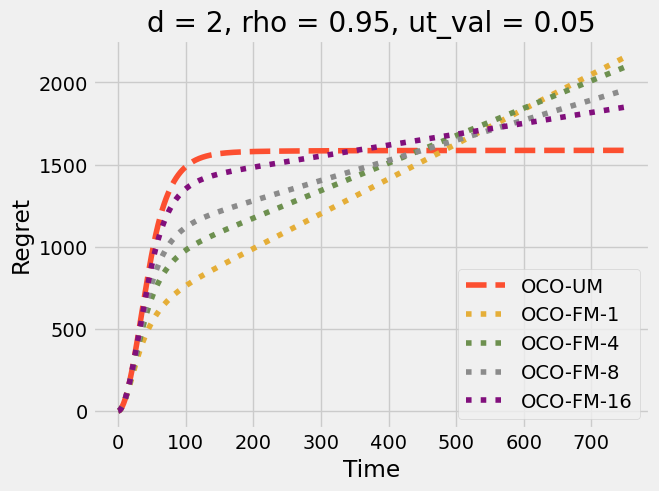}
      \caption{}
      \label{fig:d_2_rho_0.95_ut_0.05}
  \end{subfigure}
  \hfill
  \begin{subfigure}[b]{0.4\textwidth}
      \centering
      \includegraphics[width=\textwidth]{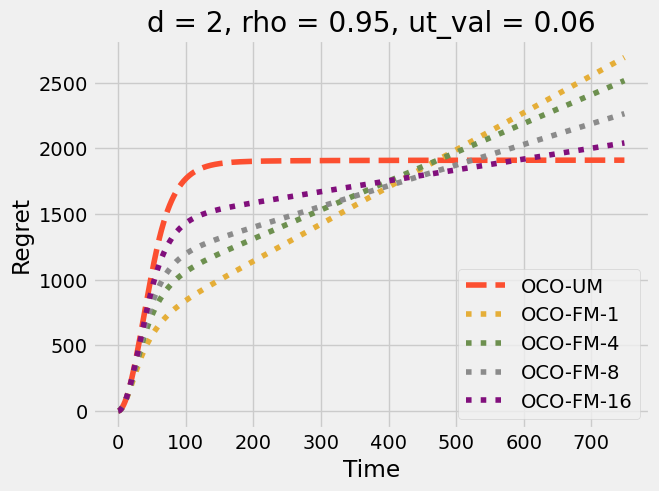}
      \caption{}
      \label{fig:d_2_rho_0.95_ut_0.06}
  \end{subfigure}
  \hfill
  \begin{subfigure}[b]{0.4\textwidth}
      \centering
      \includegraphics[width=\textwidth]{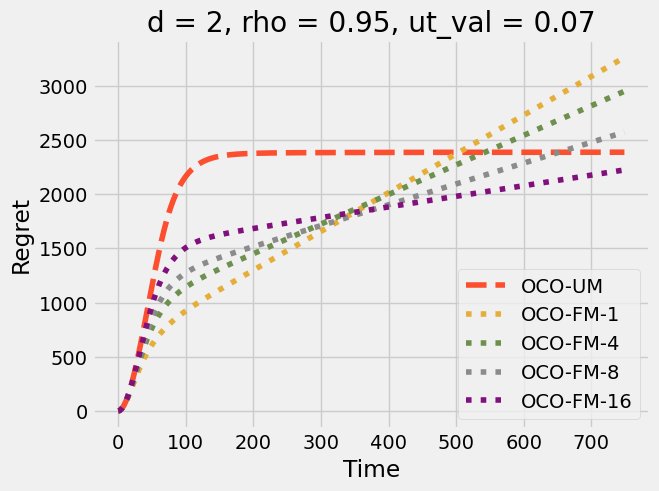}
      \caption{}
      \label{fig:d_2_rho_0.95_ut_0.07}
  \end{subfigure}
  \caption{Regret plot for $\rho = 0.95$. The label \texttt{OCO-UM} refers to formulating the problem as an OCO with unbounded memory problem. The \texttt{OCO-FM-m} refers to formulating the problem as an OCO with finite memory problem with constant memory length $m$. The titles of the plots indicate the values of the dimension, the diagonal entries of $F$, and the upper triangular entries of $F$.}
  \label{fig:results_olc_rho_0.95}
\end{figure}




\end{document}